\documentclass[preprintnumbers,floatfix,superscriptaddress,pra,onecolumn,showpacs,notitlepage,longbibliography]{revtex4-1}
\usepackage[utf8]{inputenc}
\usepackage[T1]{fontenc}
\usepackage[sc,osf]{mathpazo}
\usepackage{amsmath, amsthm, amssymb,amsfonts,mathbbol,amstext}
\usepackage{graphicx}
\usepackage{dcolumn}
\usepackage{bm}
\usepackage{bbm}
\usepackage{hyperref}
\usepackage{comment}
\usepackage{multirow}

\newtheorem{Proposition}{Proposition}
\newtheorem{Lemma}{Lemma}
\newtheorem{Corollary}{Corollary}

\newcommand{\Sp}{{\mathrm{Sp}}}
\newcommand{\rank}{{\mathrm{rank}}}

\begin{document}
\title{Semidefinite tests for latent causal structures}
\date{\today}

\author{Aditya Kela}
\affiliation{Institute for Theoretical Physics, University of Cologne, 50937 Cologne, Germany}
\author{Kai von Prillwitz}
\affiliation{Institute for Chemistry and Biology of the Marine Environment, University of Oldenburg, 26111 Oldenburg, Germany}
\author{Johan {\AA}berg}
\email{johan.aberg@uni-koeln.de}
\affiliation{Institute for Theoretical Physics, University of Cologne, 50937 Cologne, Germany}
\author{Rafael Chaves}
\affiliation{International Institute of Physics, Federal University of Rio Grande do Norte, 59070-405 Natal, Brazil}
\author{David Gross}
\affiliation{Institute for Theoretical Physics, University of Cologne, 50937 Cologne, Germany}
\affiliation{Centre for Engineered Quantum Systems, School of Physics,
The University of Sydney, Sydney, NSW 2006, Australia}

\begin{abstract}
Testing whether a probability distribution is compatible with a given Bayesian network is a fundamental task in the field of causal inference, where Bayesian networks model causal relations.
Here we consider the class of causal structures where all correlations between observed quantities are solely due to the influence from latent variables. 
We show that each model of this type imposes a certain signature on the observable covariance matrix in terms of a particular decomposition into positive semidefinite components. 
This signature, and thus the underlying hypothetical latent structure, can be tested in a computationally efficient manner via semidefinite programming.
This stands in stark contrast with the algebraic geometric tools required if the full observable probability distribution is taken into account. 
The semidefinite test is compared with tests based on entropic inequalities.
\end{abstract}

\maketitle

\section{Introduction}

In spite of the primal importance of discovering causal relations in science, the statistical analysis of empirical data has historically shied away from causality. 
Only releatively recently has a rigorous theory of causality emerged
(see, for instance, \cite{Pearlbook,Spirtesbook}), showing that empirical data indeed can contain information about causation rather than mere correlation. Since then, causal inference has quickly become influential. Examples range from applications to the inference of genetic \cite{friedman2004inferring} and social networks \cite{Steeg2011}, to a better understanding of the role of causality within quantum physics \cite{Leifer2013,Fritz2012,Fritz2014,Henson2014,Chaves2015a,Piennar2014,ried2015quantum,Costa2016,horsman2016can}.

To formalize causal mechanisms it has become popular to use directed acyclic graphs (DAGs) where nodes denote random variables and directed edges (arrows) account for their causal relations. 
Central problems within this context include \emph{inference} or \emph{model selection}: `Given samples from a number of observable variables, which DAG should we associate with them?', as well as \emph{hypothesis testing}: `Can the observed data be explained in terms of an assumed DAG?'
Here, we concentrate on the latter problem and 
propose a novel solution based on the covariances that a given causal structure  gives rise to. To understand the relevance and applicability of this method it is useful to summarize the difficulties that we typically face when approaching such problems.

The  most  common  method  to  infer  the  set  of  possible DAGs  compatible with empirical observations is based  on the Markov condition and the faithfulness assumption \cite{Pearlbook,Spirtesbook}. Under these conditions, and in the case where all variables composing a given DAG can be assumed to be empirically accessible, the conditional statistical independencies implied by the graph contain all the information required to test for the compatibility of some data with the causal structure. However, for a variety of practical and fundamental reasons, we do quite generally face causal discovery in the presence of latent (hidden) variables, that is, variables that may play an important role in the causal model, but nonetheless cannot be accessed empirically. In this case we have to characterize the set of marginal probability distributions that a given DAG can give rise to. Unfortunately, as is widely recognized, generic causal models with latent variables impose highly non-trivial constraints on the possible correlations compatible with it \cite{Pitowsky1991,Pearl1995,Geiger1999,Bonet2001,Garcia2005,Kang2006,Kang2007,evans2012graphical,lee2015causal,Chaves2016,Rosset2016,wolfe2016inflation}. Although the marginal compatibility in principle can be completely characterized in terms of semi-algebraic sets \cite{Geiger1999}, it appears that the resulting tests in practice are computationally intractable beyond a few variables \cite{Garcia2005,lee2015causal}.

One possible approach to deal with the apparent intractability is to consider relaxations of the original problem, that is, to design tests that define incomplete lists of constraints (outer approximations) to the set of compatible distributions \cite{Bonet2001,Garcia2005,Kang2006,Kang2007,moritz2014discriminating,Chaves2014,Chaves2014b}. For instance, this approach has previously been considered in \cite{Chaves2014,Chaves2014b,steudel2015information,weilenmann2016non}, with tests based on entropic information theoretic inequalities; an idea originally conceived to tackle foundational questions in quantum mechanics \cite{Braunstein1988,Cerf1997,Chaves2012,FritzChaves2013,Chaves2013entropic,Chaves2015entropy,Chaves2016entropic}. Here we consider a relaxation in a similar spirit, but based on covariances rather than entropies. 

Beyond dealing with potential computational intractabilities, an additional benefit with a relaxation based on covariances is that it at most involves bipartite marginals, and it seems reasonable to expect that this would be less data-intensive than methods based on the full multivariate distribution of the observables.

\begin{figure}[h!]
 \includegraphics[width= 11cm]{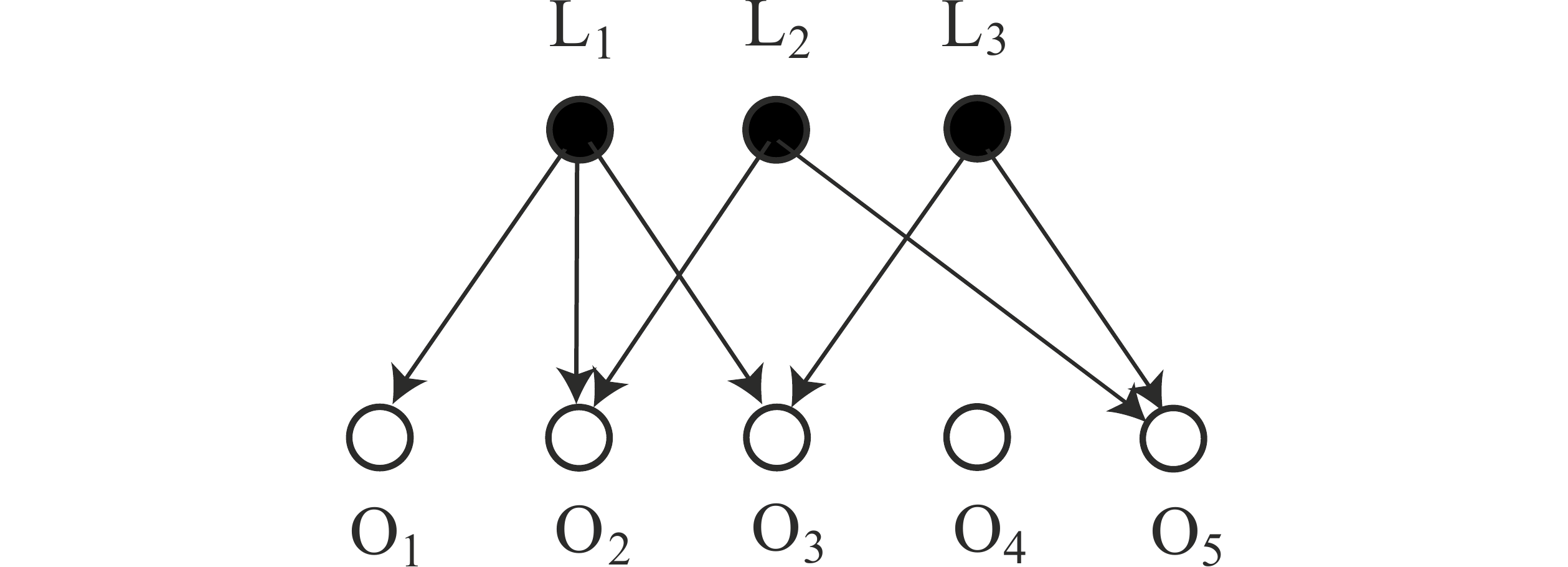} 
\caption{\label{FigBipartite} {\bf Bipartite DAGs.}  In this investigation we focus on the class of causal models where all correlations among the observables are due to a collection of independent latent variables. This setting can be described in terms of DAGs that are bipartite, where the latter means that all edges are directed from latent variables ($L_1,L_2,L_3$) to the observables ($O_1,O_2,O_3,O_4,O_5$), and where there are no edges within each of these subsets. 
}
\end{figure}

\subsection{Main assumptions and results}

We focus on a particular class of latent causal structures, where we assume that there are no direct causal influences between the observables, but only from latent variables to observables (see figure \ref{FigBipartite}). Hence, all correlations among the observables are due to the latent variables. This setting can be described by the class of DAGs where all edges are directed from latent vertices to observable vertices, but no edges within these two groups (see figure \ref{FigBipartite}). In other words, we consider the case of DAGs that are bipartite, with the coloring `observable' and `latent'.
Alternatively, this can be described in terms of hypergraphs, where each independent latent cause is associated with a hyperedge consisting of the affected observable vertices (see e.g.~\cite{evans2015graphs}).

This class of graphs has previously been considered in the context of marginalization of Bayesian networks \cite{moritz2014discriminating,steudel2015information,evans2015graphs}. They moreover provide examples of the difficulties that arise when characterizing latent structures \cite{Branciard2010,Fritz2012,Branciard2012,tavakoli2014nonlocal,Chaves2014,Chaves2014b}, where standard techniques based on the use of conditional independencies even can yield erroneous results (for a discussion, see e.g.~\cite{Spekkens2015}). 
 This type of latent structures furthermore emerges in the context of Bell's theorem \cite{Bell1964}, as well as in recent generalizations \cite{Branciard2010,Fritz2012,Branciard2012,tavakoli2014nonlocal,Chaves2016,Rosset2016,saunders2016experimental,carvacho2016experimental}, where they can be used to show that quantum  correlations between distant observers --thus without direct causal influences between them-- are incompatible with our most basic notions of cause and effect.

Irrespective of the nature of the observables (categorical or continuous) we are free to assign vectors to each possible outcome of the observables.  Our main result is to show that each bipartite DAG implies a particular decomposition of the resulting covariance matrix into positive semidefinite components. Hence, we can test whether the observed covariance matrix is compatible with a hypothetical bipartite DAG by checking whether it satisfies the corresponding positive semidefinite decomposition, and we will in the following somewhat colloquially refer to this as the `semidefinite test'.  
The semidefinite test can thus be phrased as a semidefinite membership problem, which in turn can be solved via semidefinite programming. The latter is known to be computationally efficient from a theoretical point of view, and has a good track record concerning algorithms that are efficient also in practice (see discussions in \cite{vandenberghe1996semidefinite}). 

\subsection{Structure of the paper} 
In section \ref{SecSemidefDec} we derive a general decomposition of covariance matrices, which forms the basis of our semidefinite test. In section \ref{SecObsLat} we rephrase this general result to fit with the particular structure of observables and latent variables that we employ, and in section \ref{SecDecBipartDAGs} we derive the main result, namely that every bipartite DAG implies a particular semidefinite decomposition of the observable covariance matrix. Section \ref{SecConverse} focuses on the converse, namely that every covariance matrix that satisfies the decomposition of a given bipartite DAG can be realized by a corresponding causal model. Section \ref{SecOperatorInequalities} relates the semidefinite decomposition to previous types of operator inequalities introduced in  \cite{VonPrillwitz15MasterThesis}.
To obtain a covariance matrix we may be required to assign vectors to the outcomes of the random variables, and section \ref{SecUniversalFeatureMaps} discusses the dependence of the semidefinite test on this assignment. 
In section \ref{SecMonotonicity} we briefly discuss the fact that the compatibility with a given bipartite DAG is not affected if the observables are processed locally, and that the semidefinite test respects this basic property under suitable conditions. 
Section \ref{SecMonotoneFamily} considers a specific class of distributions where it is possible to analytically determine the conditions for a semidefinite decomposition.  This class of distribution does in section  \ref{SecComparison} serve as a testbed for comparisons with the above mentioned entropic tests. We conclude with a summary and outlook in section \ref{SecSummaryOutlook}.

\section{\label{SecSemidefDec}Semidefinite decomposition of covariance matrices}
In this section we develop the basic structure that forms the core of the semidefinite test. In essence it is obtained via a repeated application of a law of total variance for covariance matrices.

For a vector-valued random variable $Y$, in a real or complex inner product space $\mathcal{V}$, we define the covariance matrix of $Y$ as 
\begin{equation}
\mathrm{Cov}(Y) := E\Big(\big(Y- E(Y)\big)  \big(Y- E(Y)\big)^{\dagger}\Big)  = E(YY^{\dagger})-E(Y)E(Y)^{\dagger},
\end{equation}
 where $E(Y)$ denotes the expectation of $Y$ and $\dagger$ denotes the transposition if the underlying vector space is real, and the Hermitian conjugation if the space is complex. 
One should note that $E(Y)^{\dagger} = E(Y^{\dagger})$.   We also define the cross-correlation for a pair of vector-valued variables $Y',Y$ (not necessarily belonging to the same vector space)
  \begin{equation}
  \label{crosscorrelation}
 \mathrm{Cov}(Y',Y) := E(Y'Y^{\dagger})-E(Y')E(Y)^{\dagger},
\end{equation}
where $\mathrm{Cov}(Y,Y) = \mathrm{Cov}(Y)$.
 For a pair of random variables $X,Y$ we denote the expectation of $Y$ conditioned on $X$ as $E(Y|X)$.
Via the conditional expectation we can also define the conditional covariance matrix 
\begin{equation}
\begin{split}
\mathrm{Cov}(Y|X) := & E\Big(\big(Y- E(Y|X)\big)  \big(Y- E(Y|X)\big)^{\dagger} \Big|X \Big)   =  E(YY^{\dagger}|X)-E(Y|X)E(Y|X)^{\dagger}.
\end{split}
\end{equation}
In a similar manner we can also obtain a conditional cross-correlation between two random vectors $Y',Y$
\begin{equation}
\begin{split}
  \mathrm{Cov}(Y',Y|X) := & E\Big(\big(Y'- E(Y'|X)\big)  \big(Y- E(Y|X)\big)^{\dagger} \Big|X \Big)   =  E(Y'Y^{\dagger}|X)-E(Y'|X)E(Y|X)^{\dagger}.
\end{split}
\end{equation}

The starting point for our derivations is the law of total expectation
\begin{equation}
\label{lawoftotalexpectation}
E(Y) = E\big(E(Y|X)\big),
\end{equation}
 where the `outer' expectation corresponds to the averaging over the random variable $E(Y|X)$. The law of total expectation can be iterated, such that for three random variables $Y,X,Z$, we have a law of total conditional expectation
\begin{equation}
\label{lawofconditionaltotalexpectation}
E(Y|Z) = E\Big(E(Y|X,Z)\Big|Z\Big),
\end{equation}
and thus $E(Y) = E\big(E(Y|Z)\big) = E\Big(E\big(E(Y|X,Z)\big|Z\big)\Big)$.

From the law of total expectation (\ref{lawoftotalexpectation}) one can obtain a covariance-matrix version of the law of total variance
\begin{equation}
\label{lawoftotalcovariance}
\mathrm{Cov}(Y)  = \mathrm{Cov}\big( E(Y|Z) \big) +  E\big( \mathrm{Cov}(Y|Z) \big),
\end{equation}
which can be confirmed by expanding the two sides of the above equality and applying (\ref{lawoftotalexpectation}).

For three random variables $Y,W,Z$ a conditional version of the law of total covariance reads
\begin{equation}
\label{lawofconditionaltotalcovariance}
\begin{split}
\mathrm{Cov}(Y|Z)  = \mathrm{Cov}\Big( E(Y|W,Z)\Big|Z\Big) +  E\Big( \mathrm{Cov}(Y|W,Z)\Big|Z\Big),
\end{split}
\end{equation}
which can be obtained by expanding the right hand side and applying the law of total conditional expectation (\ref{lawofconditionaltotalexpectation}).

The following lemma is obtained via an iterated application of the law of total covariance (\ref{lawoftotalcovariance}) and the law of  total conditional covariance (\ref{lawofconditionaltotalcovariance}). One may note the similarities with the chain-rule for entropies (see e.g.~chapter 2 in \cite{cover2012elements}).
\begin{Lemma}
\label{ChainDecomposition}
Let $Y$ be a vector-valued random variable on a finite-dimensional real or complex inner product space $\mathcal{V}$, let  $X_1,\ldots, X_N$ be random variables over the same probability space. Assuming that the underlying measure is such that all involved conditional expectations and covariances are well defined, then 
\begin{equation}
\label{zioizu}
\mathrm{Cov}(Y) = R + \sum_{n=1}^{N}C_n, 
\end{equation}
where  $R$ and $C_1,\ldots, C_N$ are positive semidefinite operators on the space $\mathcal{V}$, defined by
\begin{equation}
\label{lldavaldl}
\begin{split}
C_{1} :=  & \mathrm{Cov}\big( E(Y|X_1) \big),\\
C_{n} := & E\Big(\mathrm{Cov}\big( E(Y|X_1,\ldots,X_n)\big|X_1,\ldots,X_{n-1}\big)\Big),\quad n = 2,\ldots,N,\\
R  :=  & E\big(\mathrm{Cov}(Y|X_1,\ldots,X_N)\big).
\end{split}
\end{equation}
\end{Lemma}
One may note that the above decomposition is not necessarily unique; we could potentially obtain a new decomposition if  the variables in the sequence $X_1,\ldots, X_N$ are permuted.
\begin{proof}
The law of total covariance (\ref{lawoftotalcovariance}) for $Z=X_1$, combined with the law of total conditional covariance (\ref{lawofconditionaltotalcovariance}) for $Z:= X_1, W:= X_2$ yields 
\begin{equation}
\label{dnjvadlkv}
\begin{split}
\mathrm{Cov}(Y)  = & \mathrm{Cov}\big( E(Y|X_1) \big)  +  E\Big(  \mathrm{Cov}\big( E(Y|X_2,X_1)\big|X_1\big)    \Big) + E\big( \mathrm{Cov}(Y|X_2,X_1)\big). 
\end{split}
\end{equation}

Suppose that for some $j\geq 2$ it would be true that
\begin{equation}
\label{kjdfvlkad}
\begin{split}
\mathrm{Cov}(Y)  
 = & \mathrm{Cov}\big( E(Y|X_1) \big) \\
& +\sum_{n=2}^{j} E\bigg(\mathrm{Cov}\Big( E(Y|X_1,\ldots,X_n)\Big|X_1,\ldots,X_{n-1} \Big)\bigg)\\
& +  E\big( \mathrm{Cov}(Y|X_1,\ldots,X_j) \big).
\end{split}
\end{equation}
The law of  total conditional covariance (\ref{lawofconditionaltotalcovariance}), with $W:= X_{j+1}$ and $Z := X_1,\ldots, X_{j}$, gives
\begin{equation*}
\begin{split}
\mathrm{Cov}(Y|X_1,\ldots, X_j)  = & \mathrm{Cov}\Big( E(Y|X_1,\ldots, X_j,X_{j+1})\Big| X_1,\ldots, X_j\Big)  +  E\Big( \mathrm{Cov}(Y|X_1,\ldots, X_j,X_{j+1})\Big|X_1,\ldots, X_j\Big).
\end{split}
 \end{equation*}
  By inserting this expression into the last line of (\ref{kjdfvlkad}) one does again obtain (\ref{kjdfvlkad}) but with $j$ substituted for $j+1$. By (\ref{dnjvadlkv}) we can see that (\ref{kjdfvlkad}) is true for $j = 2$. Thus, by induction to $j= N$, and the identifications in (\ref{lldavaldl}), we obtain (\ref{zioizu}).

 Note that  $\mathrm{Cov}\big( E(Y|X_1,\ldots,X_{n-1},X_n)\big|X_1 = x_1,\ldots,X_{n-1}=x_{n-1}\big)$ is a positive semidefinite operator on $\mathcal{V}$ for each value of $x_1,\ldots,x_{n-1}$. Hence,  by averaging over these  variables, and thus implementing the expectation that yields $C_n$, we do still have a positive semidefinite operator on $\mathcal{V}$. The same observation applies to $ R = E\big( \mathrm{Cov}(Y|X_1,\ldots,X_N) \big)$.
\end{proof}

\section{\label{SecObsLat}Observable vs.\ latent variables, and feature maps}

Here we consider the decomposition developed in the previous section for the more specific setting of observable and latent variables.

We consider a collection of observable variables $O_1,\ldots,O_M$. To each of these variables $O_{m}$ we associate a mapping $Y^{(m)}$, in some contexts referred to as a `feature map' \cite{ScholkopfLearningWithKernels}, into a finite-dimensional vector space $\mathcal{V}_m$. We denote the resulting vector-valued random variables by $Y_{m}:=Y^{(m)}(O_m)$, and for the sake of simplicity we will in the following tend to abuse the terminology and refer to the vectors $Y_m$ themselves as feature maps. 
We also define the joint random vector  $Y := \sum_{m=1}^{M}Y_{m}$ on $\mathcal{V} := \bigoplus_{m=1}^{M}\mathcal{V}_m$. (Hence, we can view $Y$ as the concatenation of the vectors $Y_m$.)
One should note that while we regard the observable variables  $O_m$ as being part of the setup that is `given', the feature maps $Y^{(m)}$ are part of the analysis, and we are free to assign these as we see fit. (Concerning the question of how the test depends on this choice, see section \ref{SecUniversalFeatureMaps}.)

\begin{figure}[h!]
 \includegraphics[width= 11cm]{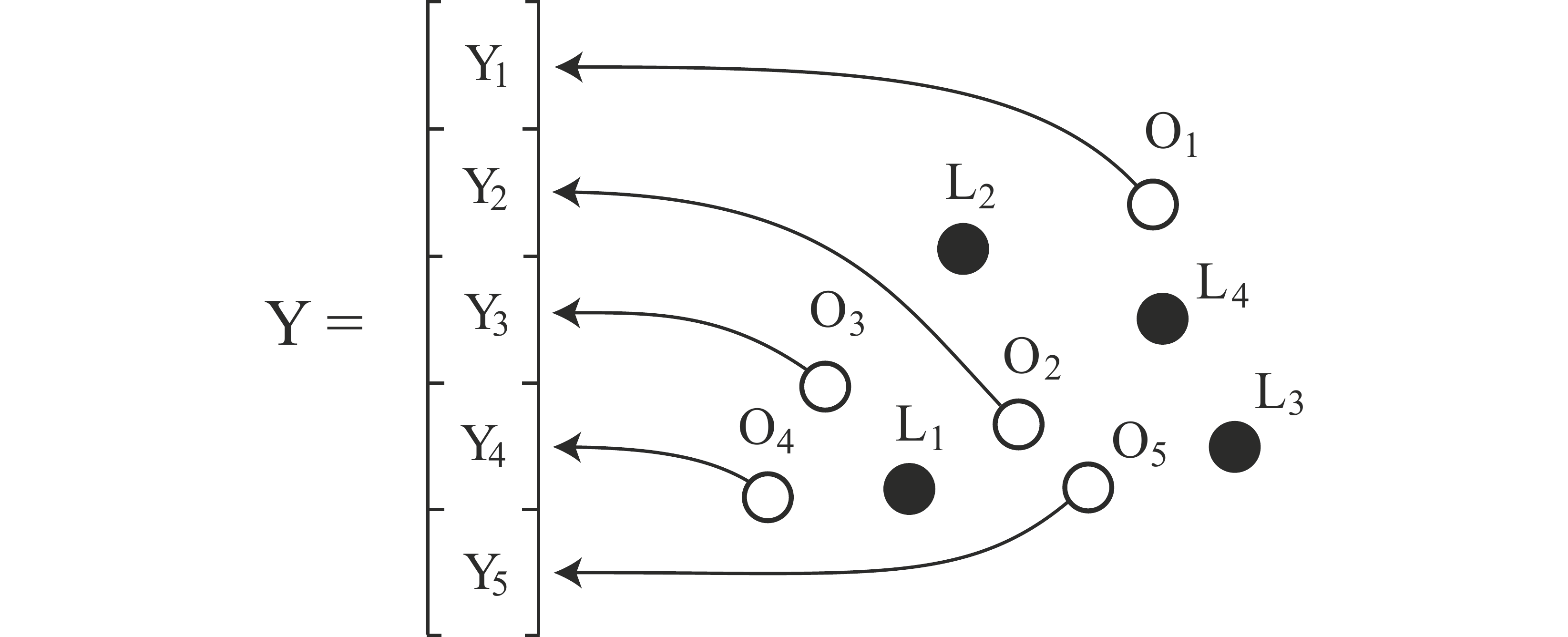} 
\caption{\label{FigObservableUnobservable} {\bf Observables, latent variables, and feature maps.} The model consists of a collection of observable variables $O_1,\ldots, O_M$ and a collection of latent variables $L_1,\ldots, L_N$. Via feature maps, each $O_m$ is mapped to a vector $Y_m$ in a vector space $\mathcal{V}_m$. On the vector space $\mathcal{V} = \bigoplus_{m=1}^{M}\mathcal{V}_m$ we define the joint random vector $Y := Y_1+\cdots + Y_M$.}
\end{figure}

Let $P_m$ denote the projector onto the subspace $\mathcal{V}_m$ in $\mathcal{V}$.
We divide the total covariance matrix $\mathrm{Cov}(Y)$ into the cross-correlations between the separate observable quantities
$\mathrm{Cov}(Y) = [\mathrm{Cov}(Y_{m},Y_{m'})]_{m,m'=1}^{M}$. One can note that $\mathrm{Cov}(Y_{m},Y_{m'}) = P_m\mathrm{Cov}(Y)P_{m'}$.

For a collection of latent variables $L_1,\ldots,L_N$, we make the identifications $X_j:=L_j$ in Lemma \ref{ChainDecomposition}.
Similarly as for the covariance matrix we decompose the operators  $C_n$ and $R$ into `block-matrices'
$C_n = [C_{n}^{m,m'}]_{m,m' = 1}^{M}$ and $R = [R^{m,m'}]_{m,m'=1}^{M}$, with $C_n^{m,m'} := P_m C_n P_{m'}$ and $R^{m,m'} := P_m R P_{m'}$, where we can write
\begin{equation}
\label{BlockForm}
\begin{split}
C_{1}^{m,m'}  = & \mathrm{Cov}\big( E(Y_m|L_1),\, E(Y_{m'}|L_1)  \big),\\
C_{n}^{m,m'}   = & E\bigg(\mathrm{Cov}\Big( E(Y_m|L_1,\ldots,L_n),\, E(Y_{m'}|L_1,\ldots,L_n)  \Big|L_1,\ldots,L_{n-1}\Big)\bigg),\\
R^{m,m'}    = & E\big(\mathrm{Cov}(Y_m,\, Y_{m'}|L_1,\ldots,L_N)\big),
\end{split}
\end{equation}
for $2\leq n\leq N$.
In terms of these blocks we can thus reformulate (\ref{zioizu}) as
\begin{equation}
\mathrm{Cov}(Y_{m},Y_{m'}) = R^{m,m'} + \sum_{n=1}^{N}C_n^{m,m'}.
\end{equation}
One should keep in mind that $C_n^{m,m'}$ and $R^{m,m'}$ in the general case are matrices (rather than scalar numbers) for each single pair $m,m'$.

\section{\label{SecDecBipartDAGs}Decomposition of the covariance matrix for bipartite DAGs}

We define a bipartite DAG as a finite DAG $G = (V,E)$ with vertices $V$ and edges $E$, with a bipartition $V = O\cup L$,  $O\cap L = \emptyset$ such that all edges in $E$ are directed from the elements in $L$ (the latent variables) to the elements in  $O$ (the observables). Since $G$ is finite, we enumerate the elements of $O$ as  $O_1,\ldots,O_M$ and the elements of $L$ as $L_1,\ldots, L_N$. One may note that we generally will overload the notation and let $O_m$ and $L_n$ denote the vertices in the underlying bipartite DAG, as well as denoting the random variables associated with these vertices.

For a vertex $v$ in a directed graph $G$ we let $\mathrm{ch}(v)$ denote the children of $v$, i.e., the set of vertices $v'$ for which there is an edge directed from $v$ to $v'$. We let $\mathrm{pa}(v)$ denote the parents of $v$, i.e., the set of vertices $v'$ for which there is an edge directed from $v'$ to $v$.
For bipartite DAGs an element in $L$ can only have children in $O$ (and have no parents), and an element in $O$ can only have parents in $L$ (and no children). As an example, for the bipartite DAG in figure \ref{FigBipartite} we have $\mathrm{ch}(L_1) = \{O_1,O_2,O_3\}$, $\mathrm{ch}(L_2) = \{O_2,O_5\}$, and $\mathrm{ch}(L_3) = \{O_3,O_5\}$, and $\mathrm{pa}(O_1) = \{L_1\}$, $\mathrm{pa}(O_2) = \{L_1,L_2\}$, $\mathrm{pa}(O_{3}) = \{L_1,L_3\}$, $\mathrm{pa}(O_{4}) = \emptyset$, and $\mathrm{pa}(O_5) = \{L_2,L_3\}$. 

For a causal model defined by a general DAG  $G = (V,E)$ the underlying probability distribution can be described via the Markov condition where each edge represents a direct causal influence, and thus each vertex $v$ can only be  directly influenced by its parents $\mathrm{pa}(v)$, resulting in distributions of the form $P = \Pi_{v\in V}P\big(v\big|\mathrm{pa}(v)\big)$. Hence, for a bipartite DAG  we get $P = \Pi_{m}P\big(O_m\big|\mathrm{pa}(O_m)\big)\Pi_n P(L_n)$, and thus all the latent variables are independent, and the observables are independent when conditioned on the latent variables.

As in the previous section, we map the observables $O_1,\ldots,O_M$ to vectors $Y_1,\ldots,Y_M$ in vector spaces $\mathcal{V}_1,\ldots,\mathcal{V}_M$. 
For each $n$ we define the projector $P^{(n)}$ in $\mathcal{V}$ by
\begin{equation}
\label{Pdef}
P^{(n)} := \sum_{m\in \mathrm{ch}(L_n)}P_{m}.
\end{equation}
Hence, $P^{(n)}$ is the projector onto all subspaces of $\mathcal{V}$ that are associated with the children $\mathrm{ch}(L_n)$ of the latent variable $L_n$. (In the above sum we should strictly speaking write $\sum_{m:O_m\in \mathrm{ch}(L_n)}$. However, in order to  avoid a too cumbersome notation we will from time to time take the liberty of writing $m\in \mathrm{ch}(L_n)$ rather than $O_m\in \mathrm{ch}(L_n)$, and $n\in \mathrm{pa}(O_m)$ rather than $L_n\in \mathrm{pa}(O_m)$.)

\begin{figure}[h!]
 \includegraphics[height= 3.5cm]{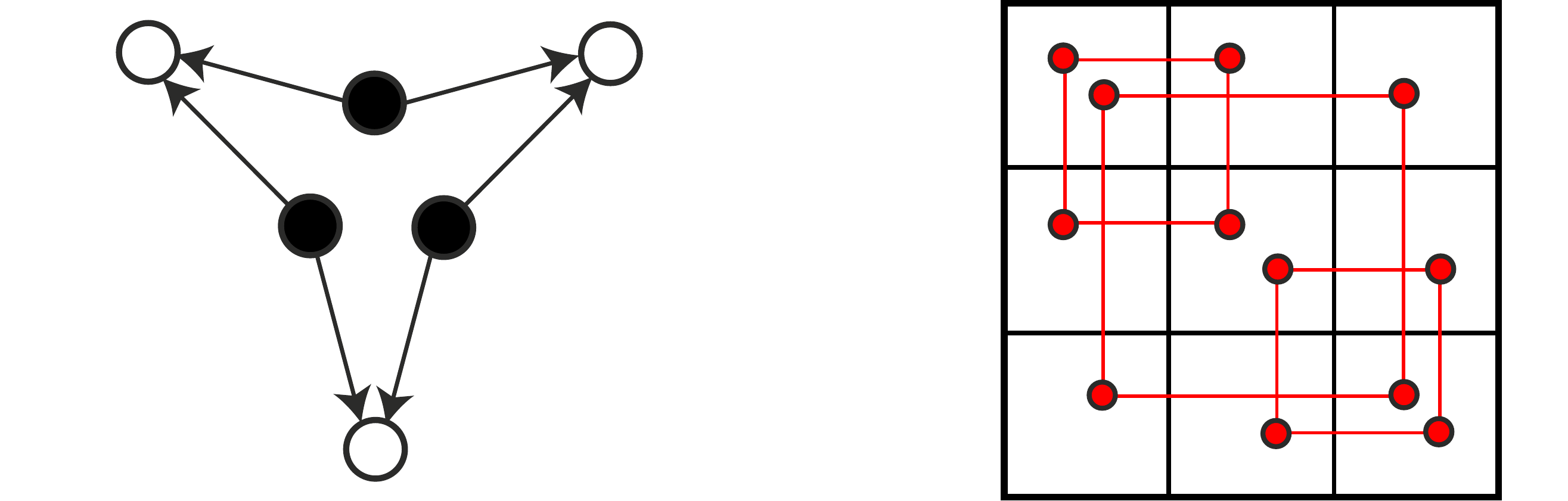} 
\caption{\label{FigTriangle} {\bf Example: Triangular bipartite DAG.}
The covariance matrix resulting from the observables in a bipartite DAG is subject to a decomposition where each latent variable gives rise to a positive semidefinite component, and where the support of that component is determined by the children of the corresponding latent variable. In the case of the `triangular' scenario of the the bipartite DAG to the left, each of the three latent variables has two children. The covariance matrix, schematically depicted to the right,  can consequently be decomposed into three positive semidefinite components, each with bipartite supports.  This observation yields a method  (which we refer to as the `semidefinite test') to falsify a given bipartite DAG as an explanation of an observed covariance matrix.
}
\end{figure}

\begin{Proposition}
\label{PropDecomposition}
For a bipartite DAG with latent variables $L_1,\ldots, L_N$ and observables $O_1,\ldots, O_M$  with assigned feature maps $Y_1,\ldots,Y_M$ into finite-dimensional real or complex inner-product spaces $\mathcal{V}_1,\ldots,\mathcal{V}_M$, the covariance matrix of $Y = \sum_{m=1}^{M}Y_m$ satisfies 
\begin{equation}
\label{poeto}
\mathrm{Cov}(Y) = R  + \sum_{n=1}^{N}C_n,\quad R\geq 0,\quad C_n\geq 0,
\end{equation}
where
\begin{equation}
\label{acdsjd}
P^{(n)}C_n P^{(n)} = C_n,\quad R = \sum_{m=1}^{M}P_m R P_m.
\end{equation}
and where the projectors $P^{(n)}$ are as defined in (\ref{Pdef}) with respect to the given bipartite DAG, and where $P_m$ is the projector onto $\mathcal{V}_m$ in $\bigoplus_{m=1}^{M}\mathcal{V}_m$. 
\end{Proposition}
One may note that if the span of the supports of $\{P^{(n)}\}_{n=1}^{N}$ covers  $\mathcal{V}$, then we can distribute the blocks $P_m RP_m$ of $R$ and add them to the different $C_n$ in such a way that the new operators still are positive semidefinite and satisfy the support structure of the original $C_n$s. The exception is if there is some observable that has no parent (as $O_4$ in figure \ref{FigBipartite}).

\begin{proof}
Select an enumeration $L_{1},\ldots, L_N$ of the latent variables. By Lemma \ref{ChainDecomposition} we know that the covariance matrix $\textrm{Cov}(Y)$ can be decomposed as in (\ref{zioizu}) with the positive semidefinite operators $R$ and $C_n$ as defined in (\ref{lldavaldl}).
In the following we will make use of the block-decomposition $C_n= [C^{m,m'}_n]_{m,m'=1}^M$  and $R = [R^{m,m'}]_{m,m'=1}^M$ with respect to the subspaces $\mathcal{V}_1,\ldots,\mathcal{V}_M$ as in (\ref{BlockForm}).

If $L_n\notin \mathrm{pa}(O_{m})$ then it means that $Y_m$ is independent of $L_n$ and thus
\begin{equation*}
\begin{split}
E(Y_m|L_1,\ldots,L_n) = E(Y_m|L_1,\ldots,L_{n-1}).
\end{split}
\end{equation*}
The analogous statement is true if $L_{n}\notin \mathrm{pa}(O_{m'})$. 
By this it follows that
\begin{equation}
\label{madflbm}
\begin{split}
& \mathrm{Cov}\Big( E(Y_m|L_1,\ldots,L_n),\, E(Y_{m'}|L_1,\ldots,L_n)  \Big|L_1,\ldots,L_{n-1}\Big) = 0,\quad \mathrm{if}\quad L_n\notin \mathrm{pa}(O_{m})\cap\mathrm{pa}(O_{m'}).
\end{split}
\end{equation}

Note that 
$ L_n\in \mathrm{pa}(O_{m})\cap \mathrm{pa}(O_{m'})\,\,\Leftrightarrow\,\, O_m,O_{m'}\in \mathrm{ch}(L_n)$.
By comparing (\ref{madflbm}) with (\ref{BlockForm}) we can conclude that 
$C^{m,m'}_n = 0$ if $O_m\notin \mathrm{ch}(L_n)$ or $O_{m'}\notin \mathrm{ch}(L_n)$.
The definition of the projector $P^{(n)}$ in (\ref{Pdef}) thus yields 
$P^{(n)}C_nP^{(n)} = C_n$. Moreover, we know from Lemma \ref{ChainDecomposition} that $C_n\geq 0$.

By construction, all the observables $O_1,\ldots,O_M$ and thus also $Y_1,\ldots,Y_M$ are independent when conditioned on the latent variables. Hence, 
\begin{equation*}
R^{m,m'}  =  E\big(\mathrm{Cov}(Y_m,\,Y_{m'}|L_1,\ldots,L_N )\big)= \delta_{m,m'}E\big(\mathrm{Cov}(Y_m|L_1,\ldots,L_N)\big),
\end{equation*}
and thus $R = \sum_{m=1}^{M}P_m RP_m$.

One may note that although the operators $C_n$ potentially may change if we generated them via a permutation of the sequence of latent variables $L_1,\ldots,L_N$, the resulting projectors $P^{(n)}$ would not change. Hence, the support-structure described by (\ref{poeto}) and (\ref{acdsjd}) is stable under rearrangements of the sequence.
\end{proof}

Deciding whether a given matrix is of the form (\ref{poeto}) can be done via semi-definite programming (SDP).
We end this section by describing an explicit SDP formulation.

The optimization will be over matrices $Z$ which can be interpreted as the direct sum of candidates for $R$ and the $C_n$'s.
More precisely, let
\begin{eqnarray}
	\mathcal{Z} &:=& 
	\mathcal{V}_1 \oplus \dots \oplus \mathcal{V}_M \oplus 
	\mathcal{W}_1 \oplus \dots \oplus \mathcal{W}_N, \label{eqn:Zblocks} \\
	\mathcal{W}_i &:=& \bigoplus_{m\in\mathrm{ch}(L_p)} \mathcal{V}_m.
	\label{eqn:Wblocks}
\end{eqnarray}
Let $Z$ be a matrix on $\mathcal{Z}$. 
According to the direct sum decomposition (\ref{eqn:Zblocks}), the matrix $Z$ is a block matrix with $(M+N)\times (M+N)$ blocks.
We think of the fist $M$ diagonal blocks as carrying candidates for $R_m=P_m R P_m$ (which completely defines $R$, according to (\ref{acdsjd})); while the rear $N$ diagonal blocks correspond to candidate $C_n$'s.
Note that the $N$ rear sumands in (\ref{eqn:Zblocks}) are dirct sums themselves.
It therefore makes sense to use double indices to refer to spaces inside the $\mathcal{W}_i$'s.
Concretely, the SDP includes affine constraints on the blocks 
$Z^{(M+n,m),(M+n,m')}$. 
The first part 
of the indices selects the space $\mathcal{W}_n$ in (\ref{eqn:Zblocks}). 
The second part 
refers to the space $\mathcal{V}_m$ within $\mathcal{W}_{n}$ according to (\ref{eqn:Wblocks}).
We use the convention that $Z^{(M+n,m),(M+n,m')}$ denotes $0$ if either $\mathcal{V}_{m}$ or $\mathcal{V}_{m'}$ does not occur in $\mathcal{W}_n$. 

With these definitions, the semi-definite program that verifies whether a covariance matrix $\mathrm{Cov}(Y)$ is of the form (\ref{poeto}) reads
\begin{eqnarray}
	\text{maximize}\quad && 0  \label{eqn:primal}\\
	\text{subject to}\quad && 
	\delta_{m,m'}
	\sum_{m=1}^M Z^{(m),(m)}
	+
	\sum_{n=1}^N Z^{(M+n,m),(M+n,m')} 
	= \mathrm{Cov}(Y)^{m,m'}, \quad (m,m'=1, \dots M) \label{eqn:affineConstraints}\\
	&& Z \geq 0,
\end{eqnarray}
where the optimization is over symmetric (hermitian) matrices $Z$ on $\mathcal{Z}$.
Up to a trivial re-expression of the linear functions of $Z$ in terms of trace inner products with suitable matrices $F_i$, the optimization problem above is in the (dual) standard form of an SDP \cite[Section~3]{vandenberghe1996semidefinite}. 

The left-hand side of (\ref{eqn:affineConstraints}) impliclity defines a linear map $\mathcal{A}$ from matrices on $\mathcal{Z}$ to matrices on $\mathcal{V}$.
Explicitly, $\mathcal{A}$ maps off-diagonal blocks to $0$ and acts on block-diagonal matrices as
\begin{equation*}
	\mathcal{A}: 
	R_1 \oplus \dots \oplus R_M \oplus C_1 \oplus \dots \oplus C_N
	\mapsto
	\sum_m R_m + \sum_n C_n.
\end{equation*}
The constraints of the SDP can thus be written slightly more transparently as 
\begin{eqnarray}
	&&\mathcal{A}(Z) = \mathrm{Cov}(Y), \label{eqn:primalA} \\
	&&Z \geq 0
\end{eqnarray}

In this language, the dual of the above SDP is
\begin{eqnarray}
	\text{minimize}\quad && \mathrm{tr}\,\big(X\,\mathrm{Cov}(Y)\big) \label{eqn:dual} \\
	\text{subject to}\quad && 
	\mathcal{A}^\dagger(X) \geq 0.
\end{eqnarray}
Let $X^\star$ be the optimizer of (\ref{eqn:dual}).
If $\mathrm{tr}\big(X^\star\,\mathrm{Cov}(Y)\big)<0$, then the original SDP is infeasible and therefore, $\mathrm{Cov}(Y)$ is not of the form (\ref{poeto}).
Indeed, by construction, such an $X^\star$ has a negative trace inner product with the covariance matrix, but a positive trace inner product 
\begin{equation*}
	\mathrm{tr}\,\big(\mathcal{A}(Z) X\big)
	=
	\mathrm{tr}\,\big(Z \mathcal{A}^\dagger(X)\big)
	\geq 
	0
	\qquad
	\forall Z \geq 0
\end{equation*}
with all matrices $\mathcal{A}(Z), Z\geq 0$ that could potentially be feasible for the primal SDP (\ref{eqn:primalA}).
Thus, the dual SDP (\ref{eqn:dual}) can be used to find a \emph{witness} or a \emph{dual certificate} $X^\star$ for the incompatibility of a covariance matrix with a presumed causal structure.
The geometry of the involved objects is shown in figure~\ref{fig:witness}.
We will refer to this dual construction in section~\ref{SecSummaryOutlook}, where we sketch possibilities to base statistical hypothesis tests such witnesses.

\begin{figure}[h!]
 \includegraphics[width=8cm]{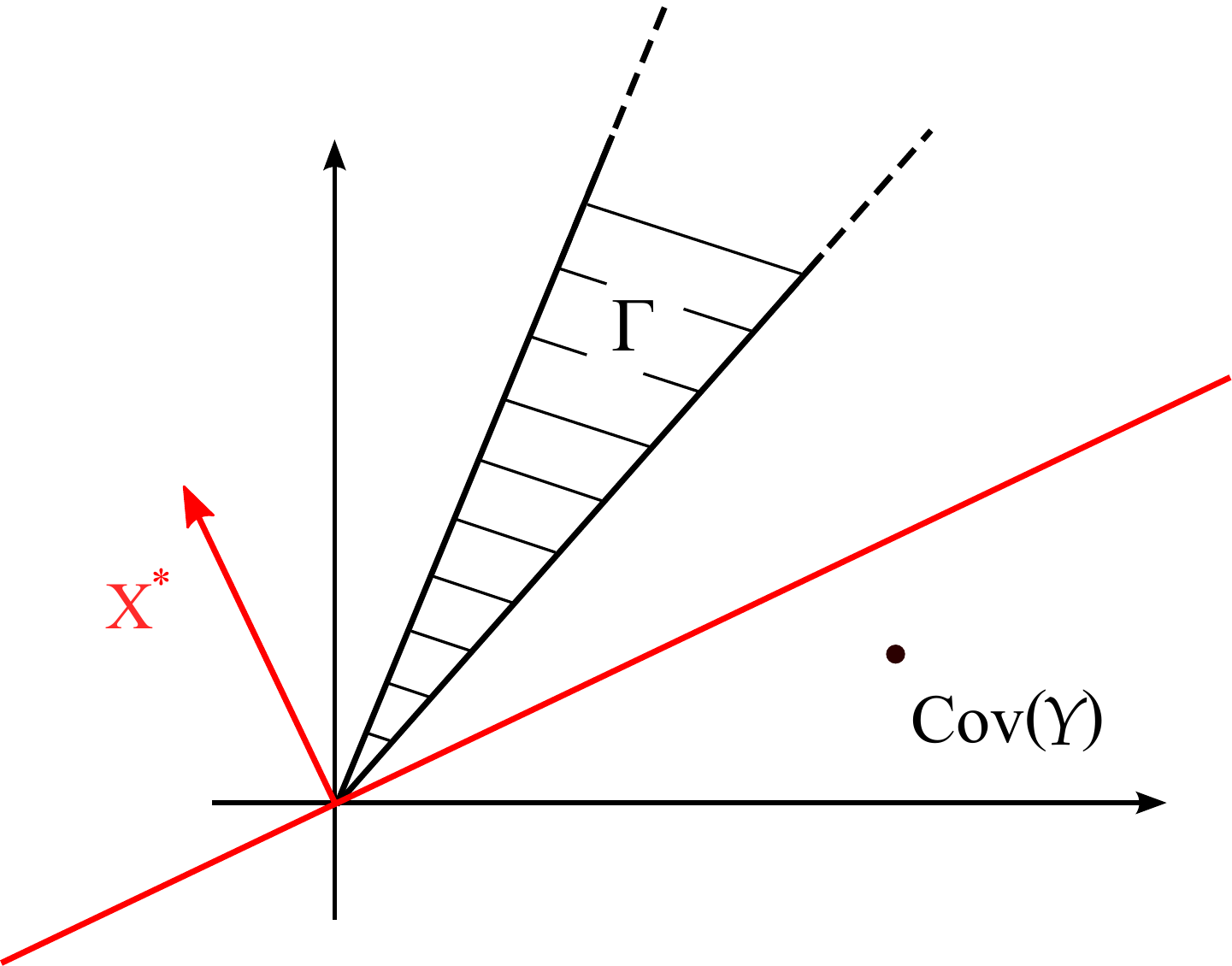} 
\caption{\label{fig:witness} {\bf Dual Certificates.}  
	The set of covariance matrices compatible with a certain causal structure in the sense of proposition~\ref{PropDecomposition} forms a convex cone $\Gamma$. 
	The cone is the feasible set of the SDP (\ref{eqn:primal}).
	If a given covariance matrix $\mathrm{Cov}(Y)$ is \emph{not} an element of that cone, then there exists a hyperplane (depicted in red) seperating the two convex sets.
	A normal vector $X^\star$ for the seperating hyperplane can be found using the dual SDP (\ref{eqn:dual}).
}
\end{figure}

\section{\label{SecConverse}Realizing a given decomposition}

In the previous section we have shown that the observable covariance matrix associated with a given bipartite DAG always satisfies a  particular semidefinite decomposition implied by that DAG. Here we show the converse, in the sense that if we have a positive semidefinite operator that satisfies the decomposition obtained from a particular bipartite DAG, then there exists a causal model  associated with that DAG that has the given operator as its observable covariance matrix  (see figure \ref{FigConverse}).
The proof is based on the observation that each positive semidefinite operator on a vector space can be interpreted as the covariance of a  vector-valued random variable on that space (e.g.~as the covariance of a multivariate normal distribution, or of variable over finite alphabets, as discussed in section \ref{SecRealPosdefCovm}). The essential idea is that we assign an independent random variable to each component in the decomposition, and take these as the latent variables, and that the support structure of the components furthermore determines the children of the latent variables.

\subsection{\label{SecRealDecompo}Realization of decompositions}

Let $O$ be a finite set, and let $\{\Omega_n\}_{n=1}^{N}$ be a collection of subsets of $O$. The collection $\{\Omega_n\}_{n=1}^{N}$  defines a bipartite DAG with $O$ as observable nodes, and a set of latent nodes $L_1,\ldots, L_N$, with the edges assigned by the identification $\mathrm{ch}(L_n):=\Omega_n$ for $n= 1,\ldots,N$. In the following we denote this bipartite DAG by $B(\{\Omega_n\}_{n=1}^{N})$.

\begin{figure}[h!]
 \includegraphics[width= 11cm]{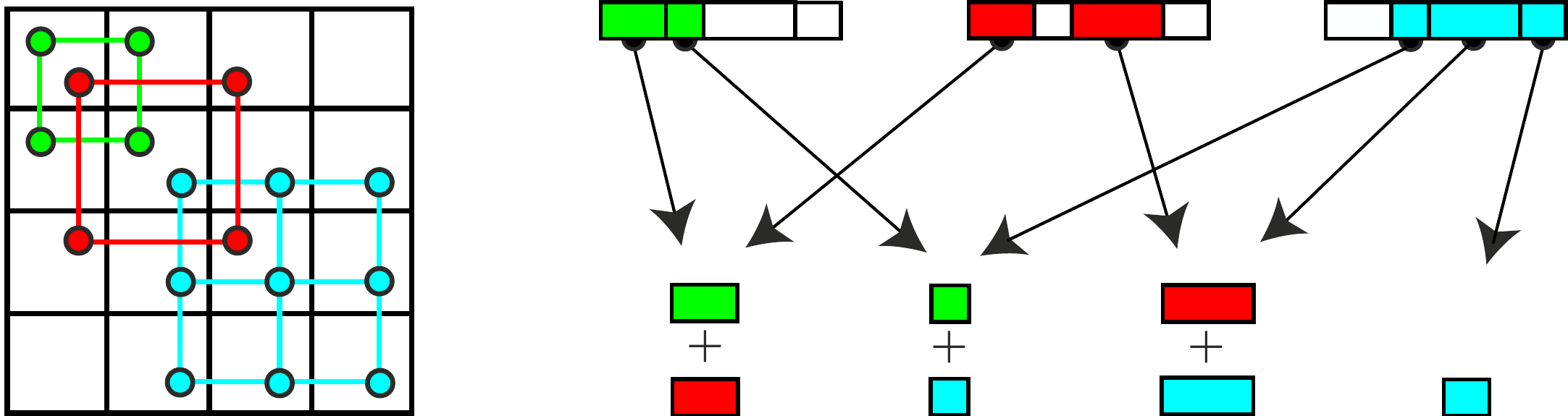} 
\caption{\label{FigConverse} {\bf } A positive semidefinite operator on a set of selected orthogonal subspaces can be regarded as the covariance matrix of a corresponding collection of vector-valued variables.  If this operator separates into positive semidefinite components (as schematically depicted to the left), then the support structures of these components define a bipartite DAG (on the right). The components in the decomposition can be interpreted as the covariance matrices of independent vector-valued latent variables. Moreover, the collection of subspaces on which such an operator has support determines the observable children of the corresponding latent variable. Each observable variable can be constructed by adding the components collected from its parents. 
}
\end{figure}

\begin{Proposition}
\label{PropRealDecomp}
Let $\mathcal{V}_1,\ldots,\mathcal{V}_M$ be finite-dimensional  real or complex inner-product spaces. 
For a number $N$ let $\{\Omega_n\}_{n=1}^{N}$ be a collection of subsets $\Omega_n\subset \{1,\ldots,M\}$. 
Suppose that $Q$ is a positive semidefinite operator on the space $\mathcal{V} = \mathcal{V}_1\oplus\cdots\oplus \mathcal{V}_M$, and that it can be written
\begin{equation}
Q = R + \sum_{n=1}^NC_n,\quad P^{(n)}C_nP^{(n)} = C_n, \quad R\geq 0, \quad C_n\geq 0,
\end{equation}
for 
\begin{equation}
\label{nvfdalavnl}
P^{(n)} =\sum_{m\in \Omega_n}P_m,\quad R = \sum_{m=1}^{M}P_m R P_m,
\end{equation}
with $P_m$ being the projectors onto the subspaces $\mathcal{V}_m$. Then there exists a causal model for the bipartite DAG $B(\{\Omega_n\}_{n=1}^{N})$ with vector-valued variables $Y_1,\ldots, Y_M$ in $\mathcal{V}_1,\ldots,\mathcal{V}_M$ such that $Y = Y_1+\cdots +Y_M$ satisfies 
\begin{equation}
\mathrm{Cov}(Y) = Q.
\end{equation}
\end{Proposition}

\begin{proof}
Let us define the set $\Omega := \cup_{n=1}^{N}\Omega_n$ and its complement $\Omega^{c} := \{1,\ldots,M\}\setminus \Omega$. By construction, 
$\Omega^c$ is the set of observable nodes in the bipartite DAG $B(\{\Omega_{n}\}_{n=1}^{N})$ that have no parents (like vertex $4$ in figure \ref{FigBipartite}) and thus each element in $\Omega$ has at least one parent. By the definition of $P^{(n)}$ in (\ref{nvfdalavnl}) it follows that $\sum_{m'\in \Omega^c}P_{m'} C_n = C_n \sum_{m'\in \Omega^c}P_{m'}  = 0$. In other words, the operators $C_n$ have no support on the subspaces belonging to parentless observable nodes. Let us now turn to the operator $R$ and its block diagonal decomposition $R = \sum_{m=1}^{M}R_m$ with $R_m := P_mRP_m$. We can write $R = \sum_{m'\in \Omega^c}R_{m'} + \sum_{m\in \Omega}R_{m}$.  Consequently, $Q$ can be decomposed in one operator $\sum_{m\in \Omega}R_{m} + \sum_{n=1}^NC_n$ on the subspace $\bigoplus_{m\in \Omega}\mathcal{V}_m$, and a collection of blocks $\{R_{m'}\}_{m'\in \Omega^c}$ on the corresponding subspaces $\mathcal{V}_{m'}$ for $m'\in \Omega^c$. Since $R_{m'}$ is positive semidefinite, it can be interpreted as the covariance matrix of some random vector $Y_{m'}$ in $\mathcal{V}_{m'}$. In the following we assume that we have made such an assignment for all $m'\in \Omega^{c}$. We also assume that these random vectors are independent. 

Each $R_m$ for $m\in \Omega$ has its support inside the support of at least one $C_n$. Hence, we can `distribute' the operators $R_m$ for $m\in \Omega$ by forming new positive semidefinite operators $\tilde{C}_n\geq 0$ such that 
\begin{equation}
\label{nvnklva}
\sum_{m\in \Omega}R_m + \sum_{n=1}^{N}C_n = \sum_{n=1}^{N}\tilde{C}_n =: \tilde{Q},  
\end{equation}
where one may note that $Q = \sum_{m'\in \Omega^c}R_{m'} + \widetilde{Q}$.

In the following we shall assign observable and latent random variables to the vertices of the bipartite DAG $B(\{\Omega_{n}\}_{n=1}^{N})$.
For each $n\in \{1,\ldots,N\}$ and each $m\in \Omega$, let $\mathcal{L}^{n}_m$ be a vector space that is isomorphic to $\mathcal{V}_m$, and let  $\phi^{n}_m:\mathcal{L}^{n}_m \rightarrow \mathcal{V}_m$ be an arbitrary isomorphism. (We assume that these isomorphisms preserve the inner-product structure, such that $\phi^n_m$ maps orthonormal bases of $\mathcal{L}^n_m$ to orthonormal bases of $\mathcal{V}_m$.) We regard the spaces in the collection $\{\mathcal{L}^{n}_m\}_{m\in\Omega, n=1,\ldots, N}$ as being orthogonal to each other.  Define $\mathcal{L}^{n} : = \bigoplus_{m\in \Omega}\mathcal{L}^{n}_{m}$, and the corresponding isomorphism $\phi^{n} := \sum_{m\in \Omega}\phi^{n}_m$. 
Since each $\tilde{C}_n$ is positive semidefinite, it can be interpreted as the covariance matrix of a vector-valued random variable on $\bigoplus_{m\in \Omega}\mathcal{V}_m$. Consequently, we can also find a vector-valued random variable $L_n$ on $\mathcal{L}^{n}$ such that 
\begin{equation}
\label{nlkalkn}
\begin{split}
\tilde{C}_n =& \mathrm{Cov}(\phi^{n}L_n)=  \phi^{n}\mathrm{Cov}(L_n){\phi^{n}}^{\dagger}.
\end{split}
\end{equation}
We assume that the random variables $L_1,\ldots,L_N$ are independent of each other, and also independent of $\{Y_{m'}\}_{m'\in \Omega^c}$.

The variables $L_1,\ldots,L_N$ serve as the latent variables corresponding to the latent nodes in the bipartite DAG $B(\{\Omega_{n}\}_{n=1}^{N})$. In the following we shall construct a collection of vector-valued variables $\{Y_{m}\}_{m\in \Omega}$ as deterministic functions of the latent variables $L_1,\ldots,L_N$, in such a way that these functions correspond to the arrows in $B(\{\Omega_{n}\}_{n=1}^{N})$, thus guaranteeing a valid causal model associated with this bipartite DAG.

Let us decompose the vector $L_n$ into its projections $L^{n}_m$ onto the subspaces $\mathcal{L}^{n}_m$.  For each $m\in \Omega_n =  \mathrm{ch}(L_n)$, the vector $L^{n}_m$ is associated to the observable node $O_m$. (One can imagine it to be transferred to node $O_m$.) Equivalently we can say that each observable node $m\in \Omega$ receives the vector $L^{n}_m$ from its ancestor $n\in \mathrm{pa}(O_m)$.
On the observable node $m\in \Omega$ we construct a new vector $Y_m$ by adding all the vectors `sent to it' from its parents
\begin{equation}
Y_m :=   \sum_{n\in \mathrm{pa}(O_m)}\phi^{n}_mL^{n}_m =  \sum_{n\in \mathrm{pa}(O_m)}\phi^{n}_mL_n
=  \sum_{n=1}^{N}\phi^{n}_mL_n,
\end{equation}
where the last equality follows since $P_mC_nP_m = 0$ if $O_m\notin \mathrm{ch}(L_n)$, or equivalently if $L_n\notin \mathrm{pa}(O_m)$, and thus $\phi^{n}_mL_n = 0$ if $n\notin \mathrm{pa}(O_m)$. 
The  collection $\{Y_{m'}\}_{m'\in \Omega^c}\cup \{Y_m\}_{m\in \Omega}$ we take as the observable variables, and we define $Y:= \sum_{m'\in \Omega^c} Y_{m'} + \sum_{m\in \Omega}Y_m =  \sum_{m'\in \Omega^c} Y_{m'} + \sum_{n=1}^{N}\phi^nL_n$.

Due to the fact that all $Y_{m'}$ for $m'\in \Omega^c$ are independent, and also independent of all $L_n$, we get
\begin{equation*}
\begin{split}
\mathrm{Cov}(Y)= & \sum_{m'\in \Omega^c}\mathrm{Cov}(Y_{m'}) + \mathrm{Cov}\Big(\sum_{n=1}^{N}\phi^nL_n, \sum_{n'=1}^{N}\phi^{n'}L_{n'}\Big)\\
= & \sum_{m'\in \Omega^c}R_{m'} +  \sum_{n,n'=1}^{N}\phi^n\mathrm{Cov}(L_n,L_{n'}){\phi^{n'}}^{\dagger}\\
&[\textrm{$L_1,\ldots,L_N$ are independent}]\\
= &  \sum_{m'\in \Omega^c}R_{m'} +\sum_{n=1}^{N}\phi^n\mathrm{Cov}(L_n){\phi^n}^{\dagger}\\
&[\textrm{By (\ref{nlkalkn})}]\\
= &  \sum_{m'\in \Omega^c}R_{m'} +\sum_{n=1}^{N}\tilde{C}_n\\
&[\textrm{By (\ref{nvnklva})}]\\
= &Q. 
\end{split}
\end{equation*}
\end{proof}

\subsection{\label{SecRealPosdefCovm}Positive semidefinite operators as covariance matrices of vector-valued random variables over finite alphabets}

The material in the previous section presumes the existence of realizations of positive semidefinite operators as the covariance of some vector-valued variable, without making any restriction on ther nature.  As mentioned above, each positive semi-definite operator (over a finite-dimensional  real or complex vector space) can be regarded as the covariance of a multivariate normal distribution. However, suppose that we would require that the variable only can take a finite number of outcomes. Here we briefly discuss the conditions for such realizations, and provide an explicit construction (in the proof of Lemma \ref{CondPosd}).

For a (possibly vector-valued) random variable over a finite alphabet, we say that that the supported alphabet size is $D$, if there are precisely $D$ outcomes that occur with a non-zero probability.
\begin{Lemma}
\label{nbklfdbnl}
If a random variable $Y$ on a finite-dimensional real or complex inner-product space has a supported alphabet size $D$, then $\rank\big(\mathrm{Cov}(Y)\big) \leq D-1$. 
\end{Lemma}
\begin{proof}
 We first note that $\mathrm{Cov}(Y) = \sum_{j=1}^{D}p_jy_jy_j^{\dagger} - \sum_{j=1}^{D}p_jy_j \sum_{j'=1}^{D}p_{j'}y_{j'}^{\dagger}$.
Since $\sum_{j=1}^{D}p_jy_j$ very manifestly is a linear combination  of $y_1,\ldots,y_D$, it follows that the range of $\sum_{j=1}^{D}p_jy_j \sum_{j'=1}^{D}p_{j'}y_{j'}^{\dagger}$ is a subset of the range of $\sum_{j=1}^{D}p_jy_jy_j^{\dagger}$, and thus $\rank\big(\mathrm{Cov}(Y)\big) \leq  \rank\big(\sum_{j=1}^{D}p_jy_jy_j^{\dagger}\big)\leq D$. However, in the following we shall show that the stronger inequality $\rank\big(\mathrm{Cov}(Y)\big) \leq D-1$ holds. To see this, let us first consider the case that $y_1,\ldots,y_D$ are linearly dependent. This means that at least one of these vectors is a linear combination of the others, and thus
 $\rank\big(\mathrm{Cov}(Y)\big) \leq D-1$. 
Let us now instead assume that $y_1,\ldots,y_D$ is a linearly independent set.  Define $Q := [Q_{j,j'}]_{j,j' =1}^{D}$ by $Q_{j,j'}:= p_j\delta_{j,j'} - p_{j}p_{j'}$, then $\mathrm{Cov}(Y) = \sum_{j,j'}y_{j}Q_{j,j'}y_{j'}^{\dagger}$. Hence, $Q$ is the matrix representation of $\mathrm{Cov}(Y)$ with respect to the linearly independent, but not necessarily orthonormal set $y_1,\ldots,y_D$. One can realize that due to the linear independence, it follows that $\rank\big(\mathrm{Cov}(Y)\big) = \rank(Q)$. Finally, let us define the $D$-dimensional vector  $\overline{1} := (1,\ldots,1)^{\dagger}/\sqrt{D}$. One can confirm that $Q\overline{1} = 0$. Hence, $\rank(Q) \leq D-1$, and we can conclude that  $\rank\big(\mathrm{Cov}(Y)\big) \leq D-1$.
\end{proof}

\begin{Lemma}
\label{CondPosd}
Let $C$ be a positive semidefinite operator on a finite-dimensional real or complex inner-product space $\mathcal{V}$. 
For every $D \geq \rank(C) +1$ there exists a vector-valued random variable $Y$ on $\mathcal{V}$ with supported alphabet size $D$, such that $C = \mathrm{Cov}(Y)$. However, $C \neq \mathrm{Cov}(Y)$ for all $Y$ with a supported alphabet size $D < \rank(C)+1$.
\end{Lemma}

\begin{proof}
Let $D$ be the supported alphabet size of a vector-valued random variable $Y$.
If  $D < \rank(C) +1$, then we know from Lemma \ref{nbklfdbnl} that $C\neq \mathrm{Cov}(Y)$. Hence, it remains to show that it is possible to find a $Y$ such that $C = \mathrm{Cov}(Y)$ for every $D \geq \rank(C)+1$. We thus wish to find a collection of vectors $y_1,\ldots,y_D\in\mathcal{V}$, and $p_1,\ldots,p_D$ with $p_j >0$, and $\sum_{j=1}^Dp_j = 1$, such that  $C = \sum_{j=1}^{D}p_jy_jy_j^{\dagger} - \sum_{j=1}^{D}p_jy_j \sum_{j'=1}^{D}p_{j'}y_{j'}^{\dagger}$.

Let $\{z_k\}_{k=1}^K$ be an orthonormal basis of the range (support) of the operator $C$, and let $P_C$ be the  projector onto the range. Let $U$ be a matrix in $\mathbb{R}^{D\times D}$ ($\mathbb{C}^{D\times D}$) if the underlying space $\mathcal{V}$ is real (complex).  Since $D\geq K+1$, we can assign the $(K+1)$th column of $U$ to be the vector $\overline{1}:=(1,\ldots, 1)^{\dagger}/\sqrt{D}$ (i.e., $U_{j,K+1} = \frac{1}{\sqrt{D}}$ for all $j= 1,\ldots, D$) and we arbitrarily complete the rest of the matrix $U$ such that it becomes orthogonal (unitary). Since $U$ is orthogonal (unitary), it follows that its columns form an orthonormal basis of $\mathbb{R}^{D}$ ($\mathbb{C}^{D}$). Hence, for each $k=1,\ldots, K$ it must be the case that the vector $(U_{j,k})_{j=1}^{D}$ is orthogonal to $\overline{1}$, and thus 
\begin{equation}
\label{weorpb}
\sum_{j=1}^{D}U_{j,k} = 0,\quad k=1,\ldots, K.
\end{equation}
Next, define the set of vectors $\{v_{j}\}_{j=1}^{D}\subset \mathcal{V}$ by $v_{j} := \sum_{k=1}^{K}U_{j,k}z_{k}$.
One can confirm that $\sum_{j=1}^{D}v_jv_j^{\dagger} =  P_C$, as well as $\sum_{j=1}^{D}v_j
= \sum_{k=1}^{K}\sum_{j=1}^{D}U_{j,k}z_{k}=  0$, where we use (\ref{weorpb}).
As the final step we define $p_j := \frac{1}{D}$ and $y_j := \sqrt{D}\sqrt{C}v_j$ for $j =1,\ldots,D$.
One can confirm that 
\begin{equation*}
\begin{split}
\sum_{j=1}^{D}p_jy_jy_j^{\dagger} = \sqrt{C}\sum_{j=1}^{D}v_jv_j^{\dagger}\sqrt{C}=  \sqrt{C}P_C\sqrt{C}=  C,\quad\quad \sum_{j=1}^{D}p_jy_j = \frac{1}{\sqrt{D}}\sqrt{C}\sum_{j=1}^{D}v_j=  0.
\end{split}
\end{equation*}
Thus, if a vector-valued random variable $Y$ takes $y_j$ with probability $p_j$, we have $\mathrm{Cov}(Y) = C$.
\end{proof}

\section{\label{SecOperatorInequalities}Implied operator inequalities}

 Here we show that the existence of positive semidefinite decompositions as in Proposition \ref{PropDecomposition} implies operator inequalities of a type studied in \cite{VonPrillwitz15MasterThesis}. 

Consider as usual a  bipartite DAG with latent variables $L_1,\ldots, L_N$ and observables $O_1,\ldots, O_M$  with assigned feature maps $Y_1,\ldots,Y_M$ into vector spaces $\mathcal{V}_1,\ldots,\mathcal{V}_M$.
For a number $d$ (whose meaning is going to be evident shortly) we define the following map on the space of operators on $\mathcal{V} = \oplus_{m=1}^M\mathcal{V}_m$ 
\begin{equation}
\label{Mapdef}
\Phi(Q) := (d-1)P_1QP_1 + \sum_{m=2}^{M}(P_mQP_m + P_1QP_m + P_mQP_1),
\end{equation}
where $P_m$ are the projectors onto the spaces $\mathcal{V}_m$ as discussed in section \ref{SecObsLat}.
Theorem 4.1 in  \cite{VonPrillwitz15MasterThesis} does in essence say that if all the latent variables $L_n$ in the given bipartite DAG have degree at most $d$, then the resulting covariance matrix $\mathrm{Cov}(Y)$ satisfies 
\begin{equation}
\label{fmbkdlmfb}
\Phi\big(\mathrm{Cov}(Y)\big)\geq 0,
\end{equation}
or if one prefers  matrix notation
\begin{equation}
\label{savlnkaf}
\left[\begin{matrix}
 (d-1)\mathrm{Cov}(Y_1) & \mathrm{Cov}(Y_1,Y_2) &  \cdots  & \cdots & \mathrm{Cov}(Y_1,Y_M)\\
\mathrm{Cov}(Y_2,Y_1) & \mathrm{Cov}(Y_2) &  0 & \cdots  & 0 \\
\vdots &  0 & \ddots  & \ddots &\vdots \\
\vdots  & \vdots & \ddots & \ddots  &  0\\
\mathrm{Cov}(Y_M,Y_1) & 0 &\cdots & 0 &  \mathrm{Cov}(Y_M) 
\end{matrix}\right]\geq 0.
\end{equation}  
Hence, by deleting a particular collection of blocks from the full covariance matrix $\mathrm{Cov}(Y)$, and adding copies of the diagonal block $\mathrm{Cov}(Y_1)$, we obtain a positive semidefinite operator. It may be worth emphasizing that mere positive semidefiniteness of $Q$ is not enough to guarantee that $\Phi(Q)$ is positive semidefinite. Hence, (\ref{fmbkdlmfb}) can indeed be used as a test of the underlying latent structure. As one may note, equations (\ref{fmbkdlmfb}) and (\ref{savlnkaf}) single out observable $1$, but by relabelling we can obtain  analogous inequalities for all observables. As an example, for the triangular scenario in figure \ref{FigTriangle}, the inequality  (\ref{savlnkaf}) and its permutations take the form
\begin{equation*}
\left[\begin{smallmatrix}
\mathrm{Cov}(Y_1) & \mathrm{Cov}(Y_1,Y_2) & 0\\
\mathrm{Cov}(Y_2,Y_1) & \mathrm{Cov}(Y_2) & \mathrm{Cov}(Y_2,Y_3)\\
0 &  \mathrm{Cov}(Y_3,Y_2) &\mathrm{Cov}(Y_3) 
\end{smallmatrix}\right]\geq 0,\quad 
\left[\begin{smallmatrix}
\mathrm{Cov}(Y_1) & \mathrm{Cov}(Y_1,Y_2) & \mathrm{Cov}(Y_1,Y_3)\\
\mathrm{Cov}(Y_2,Y_1) & \mathrm{Cov}(Y_2) & 0\\
\mathrm{Cov}(Y_3,Y_1) &  0 &\mathrm{Cov}(Y_3) 
\end{smallmatrix}\right]\geq 0,\quad 
\left[\begin{smallmatrix}
\mathrm{Cov}(Y_1) & 0 & \mathrm{Cov}(Y_1,Y_3)\\
0 & \mathrm{Cov}(Y_2) & \mathrm{Cov}(Y_2,Y_3)\\
\mathrm{Cov}(Y_3,Y_1) &  \mathrm{Cov}(Y_3,Y_2) &\mathrm{Cov}(Y_3) 
\end{smallmatrix}\right]\geq 0.
\end{equation*}

The following proposition shows that the semidefinite decomposition implies the operator inequality (\ref{fmbkdlmfb}) under the assumption that all the latent variables (regarded as vertices in a bipartite graph) have the degree at most $d$.
\begin{Proposition}
For a bipartite DAG with latent variables $L_1,\ldots, L_N$, each with degree at most $d$, and observables $O_1,\ldots, O_M$  with assigned feature maps $Y_1,\ldots,Y_M$ into finite-dimensional real or complex inner-product spaces $\mathcal{V}_1,\ldots,\mathcal{V}_M$, the covariance matrix of $Y = \sum_{m=1}^{M}Y_m$ satisfies $\Phi\big(\mathrm{Cov}(Y)\big)\geq 0$, where $\Phi$ is as defined in (\ref{Mapdef}).
\end{Proposition}

\begin{proof}
We know from Proposition \ref{PropDecomposition} that  $\mathrm{Cov}(Y) = R+\sum_{n=1}^NC_n$ with $P^{(n)}C_nP^{(n)} = C_n$, $C_n\geq 0$, and where $R$ is such that $\sum_{m}P_m R P_m = R$ and $R\geq 0$. Due to this, we have 
\begin{equation}
\label{abkdfbieuf}
\Phi(R) = (1-d)P_1RP_1 + \sum_{m=2}^{M}P_mRP_m \geq 0.
\end{equation}
 For each $C_n$ we can distinguish two cases. 

In the first case, $C_n$ has no support on $\mathcal{V}_1$, i.e., $P_1C_nP_1 = 0$. Due to the positive semidefiniteness of $C_n$ it also follows that $P_1C_nP_j = 0$ for $j = 2,\ldots,M$, and thus
\begin{equation}
\label{mlvdmld}
 \Phi(C_n) = \sum_{m=2}^{M}P_m C_n P_m \geq 0.
\end{equation}
 In the second case, $C_n$ does have a support on $\mathcal{V}_1$, meaning that $P_1C_n P_1 \neq 0$. By assumption, the latent variable $L_n$ has degree at most $d$, which means that $C_n$ has support on at most $d$ of the subspaces $\mathcal{V}_1,\ldots,\mathcal{V}_M$. Hence, apart form $\mathcal{V}_1$, there are at most $d-1$ further spaces involved. We enumerate these spaces as $\mathcal{V}_{m(2)},\ldots,\mathcal{V}_{m(d)}$, and let $\mathcal{V}_{m(1)} = \mathcal{V}_1$.  Hence, it may be the case that $P_{m(j)}C_n P_{m(j)} \neq 0$ for $j=1,\ldots, d$, while $P_{m}C_n P_m = 0$ for the remaining values of $m$. Due to the positive semidefiniteness of $C_n$, we can analogously  have $P_1C_n P_{m(j)}\neq 0$, and $P_{m(j)}C_n P_{1}\neq 0$, but $P_1C_n P_{m}= 0$, and $P_{m}C_n P_{1} =  0$ for the other values of $m$. We can conclude that 
\begin{equation}
\label{salkdfm}
\begin{split}
\Phi(C_n) = & (d-1)P_1C_nP_1 + \sum_{m=2}^{M}(P_mC_nP_m + P_1C_nP_m + P_m C_nP_1)\\
= &  (d-1)P_1C_nP_1 + \sum_{j=2}^{d}(P_{m(j)}C_nP_{m(j)} + P_1C_nP_{m(j)} + P_{m(j)}C_n P_1)\\
= &   \sum_{j=2}^{d}\Big(P_1 +P_{m(j)}\Big) C_n\Big(P_1 +   P_{m(j)}\Big)\geq 0.
\end{split}
\end{equation}
The combination of  (\ref{abkdfbieuf}), (\ref{mlvdmld}) and (\ref{salkdfm}) yields  $\Phi\big(\mathrm{Cov}(Y)\big) = \Phi(R) + \sum_{n=1}^{N}\Phi(C_n)\geq 0$, which proves (\ref{fmbkdlmfb}).
\end{proof}

We note that the operator inequalities derived here need not be tight in all cases.
Indeed, it is not hard to verify that the maps $\Phi_\alpha$ defined by
\begin{equation*}
	\Phi_\alpha: 
\left[\begin{smallmatrix}
\mathrm{Cov}(Y_1) & \mathrm{Cov}(Y_1,Y_2) & \mathrm{Cov}(Y_1,Y_3)\\
\mathrm{Cov}(Y_2,Y_1) & \mathrm{Cov}(Y_2) & \mathrm{Cov}(Y_2,Y_3)\\
\mathrm{Cov}(Y_3,Y_1) &  \mathrm{Cov}(Y_3,Y_2) &\mathrm{Cov}(Y_3) 
\end{smallmatrix}\right]
\mapsto
\left[\begin{smallmatrix}
\mathrm{Cov}(Y_1) & e^{i\alpha} \mathrm{Cov}(Y_1,Y_2) & \mathrm{Cov}(Y_1,Y_3)\\
e^{-i\alpha}\mathrm{Cov}(Y_2,Y_1) & \mathrm{Cov}(Y_2) & \mathrm{Cov}(Y_2,Y_3)\\
\mathrm{Cov}(Y_3,Y_1) &  \mathrm{Cov}(Y_3,Y_2) &\mathrm{Cov}(Y_3) 
\end{smallmatrix}\right]
\end{equation*}
preserve the set of covariance matrices compatible with the triangle scenario.
Here, $\alpha\in[0,2\pi)$ is a phase factor.
In particular, $\Phi_\alpha$ preserves positivity when acting on covariance matrices arising in this context.
This is a strictly stronger result than the one we have obtained above:
The map $\Phi$ treated in the proposition is just the equal-weight convex combination of $\Phi_\pi$ and $\Phi_0$:
\begin{equation*}
  \frac12
  \left[\begin{smallmatrix}
  \mathrm{Cov}(Y_1) & \mathrm{Cov}(Y_1,Y_2) & \mathrm{Cov}(Y_1,Y_3)\\
  \mathrm{Cov}(Y_2,Y_1) & \mathrm{Cov}(Y_2) & \mathrm{Cov}(Y_2,Y_3)\\
  \mathrm{Cov}(Y_3,Y_1) &  \mathrm{Cov}(Y_3,Y_2) &\mathrm{Cov}(Y_3) 
  \end{smallmatrix}\right]
  +
  \frac12
  \left[\begin{smallmatrix}
  \mathrm{Cov}(Y_1) & -\mathrm{Cov}(Y_1,Y_2) & \mathrm{Cov}(Y_1,Y_3)\\
  -\mathrm{Cov}(Y_2,Y_1) & \mathrm{Cov}(Y_2) & \mathrm{Cov}(Y_2,Y_3)\\
  \mathrm{Cov}(Y_3,Y_1) &  \mathrm{Cov}(Y_3,Y_2) &\mathrm{Cov}(Y_3) 
  \end{smallmatrix}\right]
  =
  \left[\begin{smallmatrix}
  \mathrm{Cov}(Y_1) & 0 & \mathrm{Cov}(Y_1,Y_3)\\
  0 & \mathrm{Cov}(Y_2) & \mathrm{Cov}(Y_2,Y_3)\\
  \mathrm{Cov}(Y_3,Y_1) &  \mathrm{Cov}(Y_3,Y_2) &\mathrm{Cov}(Y_3) 
  \end{smallmatrix}\right].
\end{equation*}

It may potentially be fruitful to consider a general theory of maps that preserve the convex cone of covariances compatible with a given causal structure.

\section{\label{SecUniversalFeatureMaps}Universal feature maps for finite categorical variables} As the reader may have realized, the choice of feature maps $Y^{(m)}$ may affect the outcome of the semidefinite test. In other words, even if we find a particular setup that is compatible with the given bipartite DAG, it may be the case that another assignment of the vectors $Y_{m}$ could yield a violation; thus potentially suggesting that we ideally should test an infinite number of choices. However, in the case of observable variables with only finite number of outcomes, we shall here see that one can make a single test, based on a sufficiently `powerful' choice of feature maps.  Suppose that the variables $O_m$ can only take a finite number of outcomes $o_1^m,\ldots, o^m_{d_m}$. An arbitrary assignment of a feature map would correspond to a collection of vectors $y^m_1,\ldots, y^m_{d_m}\in \mathcal{V}_m$ for some vector space $\mathcal{V}_m$. Now suppose that we make the additional restriction that $y^m_1,\ldots, y^m_{d_m}$ are linearly independent, and that $\dim(\mathcal{V}_m) = d_m$.  
Suppose that we have some other arbitrary assignment of feature map $\tilde{Y}_m$ given by a collection of vectors $\tilde{y}^m_1,\ldots, \tilde{y}^m_{d_m}\in \tilde{\mathcal{V}}_m$ for some vector space $\tilde{\mathcal{V}}_m$ (without any requirement of  linear independence). 
One can realize that it is always possible to find a linear map $\phi_m:\mathcal{V}_m\rightarrow \tilde{\mathcal{V}}_m$ such that $\phi_m y^m_j = \tilde{y}^m_j$, and thus $\phi_m Y_m = \tilde{Y}_m$. 
To see this, one can note that since  $y^m_1,\ldots, y^m_{d_m}$ is a linearly independent set in a $d_m$-dimensional space, it follows that the Gram matrix $G = [G_{j,j'}]_{j,j' = 1}^{d_m}$ with $G_{j,j'} := (y^m_j,y^m_{j'})$ is invertible (and positive definite). One can confirm that $\phi_m$ defined by $\phi_m(v) := \sum_{jj'}\tilde{y}^m_j[G^{-1}]_{jj'}(y^m_{j'},v)$ satisfies $\phi_m y^m_j = \tilde{y}^m_j$. In other words, a feature map with linearly independent components is `universal' in the sense that  we can generate all other feature maps on all other vector spaces, and it is moreover sufficient to do this via linear transformations.

For a collection of universal feature maps $Y_1,\ldots,Y_M$ assigned to $O_1,\ldots, O_M$, we can reach all other feature maps $\tilde{Y}_1,\ldots, \tilde{Y}_M$, by linear operations $\tilde{Y}_m = \phi_m Y_m$. Moreover, the covariance matrix $\mathrm{Cov}(Y)$ for $Y = \sum_{m=1}^{M}Y_m$ and the covariance matrix $\mathrm{Cov}(\tilde{Y})$ for $\tilde{Y} = \sum_{m=1}^{M}\tilde{Y}_m$ are related by $\mathrm{Cov}(\tilde{Y}) = \phi\mathrm{Cov}(Y)\phi^{\dagger}$ for $\phi := \sum_{m=1}^{M}\phi_m$. One can realize that if $\mathrm{Cov}(Y)$ satisfies the decomposition in Proposition \ref{PropDecomposition} for a given bipartite DAG, then $\mathrm{Cov}(\tilde{Y})$ also satisfies the decomposition. We can conclude that it is sufficient to apply the semidefinite test for a single collection of feature maps, where each of these have linearly independent components. (A convenient choice would be mappings to orthonormal bases.)
 
It is conceivable that a similar construction would hold for variables with a countably infinite number of outcomes, and it is an interesting question if one in some sense could make `universal' assignments of feature maps also in the case of a continuum. However, we shall not consider these issues in this investigation, but leave them as open questions. 

\section{\label{SecMonotonicity} Monotonicity under local operations}

Suppose that we would process each observable variable in a collection $O_1,\ldots, O_M$ `locally'. In other words,  the output $\tilde{O}_m$ is  a (possibly random) function only of $O_m$. If we restrict ourselves to discrete random variables, then this type of mapping from an input distribution $P^{M}$ of the $O_1,\ldots, O_M$, to the output distribution $\tilde{P}^{M}$ of $\tilde{O}_1,\ldots, \tilde{O}_M$ can be written
\begin{equation}
\label{fbdjkfbdk}
\tilde{P}^{M}(\tilde{x}_1,\ldots,\tilde{x}_M) := \sum_{x_1\ldots,x_M}P^1(\tilde{x}_1|x_1)\cdots P^M(\tilde{x}_M|x_M)P^{M}(x_1,\ldots,x_M),
\end{equation}
where all $P^{m}(\tilde{x}_m|x_m)$ are conditional distributions.
 From this construction it is clear that if a distribution $P^{M}$ is compatible with the given bipartite DAG, then the resulting distribution $\tilde{P}^{M}$ on $\tilde{O}_1,\ldots,\tilde{O}_M$ will also be compatible with the very same DAG. In other words, compatibility with a given bipartite DAG is in this sense a monotone with respect to local operations.

There is a priori no reason to expect that relaxations of the compatibility problem would satisfy this monotonicity. However, here we show that this property is respected by the semidefinite test, if the latter is based on universal feature maps (in the sense of the previous section). The fact that universality is needed can be seen from the following trivial special case. We assign feature maps $Y_m$ to $O_m$, and $\tilde{Y}_m$ to $\tilde{O}_m$. In principle we can for each $m$ choose all components of $Y_m$ to be identical, thus resulting in a zero covariance matrix that trivially satisfies all decompositions, while $\tilde{Y}_m$ may still result in a violation. By assuming that all the feature maps $Y_m$ are universal, we shall in the following see that monotonicity is guaranteed.

Let us first focus on the transformation of a single observable variable $O_m$ to $\tilde{O}_m$, and let us assume that $Y_m$ has the  linearly independent components $y_1^m,\ldots, y^m_K$, with Gram matrix $G = [G_{x,x'}]_{x,x' =1}^{K}$ with $G_{x,x'} = (y_{x}^m,y_{x'}^m)$, in a $K$-dimensional vector space $\mathcal{V}_m$. $G$ is invertible since $y_1^m,\ldots, y^m_K$ are linearly independent.  Let $\tilde{y}^m_1,\ldots,\tilde{y}^m_L$ be the components of $\tilde{Y}_m$ in $\tilde{\mathcal{V}}_m$. (If $L = K$ we can of course  choose $\tilde{y}^m := y^m$ as a special case.)
Define $\psi_m(v) := \sum_{\tilde{x},x',x''}\tilde{y}_{\tilde{x}}P^m(\tilde{x}|x')[G^{-1}]_{x',x''}(y_{x''},v)$. (Here and in the following  we omit the superscript `$m$' on the vectors $y$ for notational convenience.)
One can confirm that $E(\tilde{Y}_m) = \psi_m\big(E(Y_m)\big)$, and thus with $\psi =\sum_{m}\psi_m$ we get $E(\tilde{Y}) = \psi\big(E(Y)\big)$.

It may be very tempting to assume that $\mathrm{Cov}(\tilde{Y})$ would be equal to $\psi\mathrm{Cov}(Y)\psi^{\dagger}$. However, this is generally \emph{not} the case. The off-diagonal blocks for $m\neq m'$ satisfy $\mathrm{Cov}(\tilde{Y}_m,\tilde{Y}_{m'}) = \psi_m\mathrm{Cov}(Y_m,Y_{m'})\psi_{m'}^{\dagger}$.

However, for the diagonal blocks it is the case that 
\begin{equation}
\label{ndklvald}
\begin{split}
\mathrm{Cov}(\tilde{Y}_m) = & \psi_m\mathrm{Cov}(Y_m)\psi^{\dagger}_m + W_m,\\ 
W_m := & 
\sum_{\tilde{x},x}\tilde{y}_{\tilde{x}}{\tilde{y}_{\tilde{x}}}^{\dagger}P^m( \tilde{x}|x)P(O_m = x)
-\sum_{\tilde{x},\tilde{x}',x}\tilde{y}_{\tilde{x}}\tilde{y}_{\tilde{x}'}^{\dagger}P^m(\tilde{x}|x)P^m(\tilde{x}'|x)P(O_m = x).
\end{split}
\end{equation}
One can note that each `correction term' $W_m$ is supported only on the subspace $\tilde{\mathcal{V}}_m$, and one can moreover show that $W_m\geq 0$. To see the latter, let $c\in\tilde{\mathcal{V}}_m$, and define $z_{\tilde{x}} = (c,\tilde{y}_{\tilde{x}})$. Then
\begin{equation*}
(c,W_mc) = \sum_{x}P(O_m = x)\Big( \sum_{\tilde{x}}|z_{\tilde{x}}|^2P^m(\tilde{x}|x) -\Big|\sum_{\tilde{x}}   z_{\tilde{x}} P^m(\tilde{x}|x)\Big|^2\Big) = \sum_{x,\tilde{x}}P(O_m = x)P^m(\tilde{x}|x)\Big| z_{\tilde{x}} - \sum_{\tilde{x}'}P^m(\tilde{x}'|x)z_{\tilde{x}'} \Big|^2\geq 0.
\end{equation*}

If $\mathrm{Cov}(Y)$ satisfies the decomposition (\ref{poeto}) in Proposition \ref{PropDecomposition} for some bipartite DAG, then one can confirm that $\psi\mathrm{Cov}(Y)\psi^{\dagger}$ also satisfies the corresponding decomposition with respect to the subspaces $\{\tilde{\mathcal{V}}_m\}_m$. Moreover, since the correction terms $W_m$ are positive semidefinite and block-diagonal with respect to these subspaces, it follows that $\mathrm{Cov}(\tilde{Y}) = \psi\mathrm{Cov}(Y)\psi^{\dagger} + \sum_{m}W_m$ also satisfies the decomposition.
 We can thus conclude that if the initial feature maps $Y_1,\ldots, Y_M$ are universal, then the test is monotonous with respect to local operations.

As a final remark one may note that in the special case that all $P^{m}(\tilde{x}|x)$ correspond to deterministic mappings, i.e., when the output $\tilde{x}$ is a (deterministic) function of the input $x$, then  $P^{m}(\tilde{x}|x)P^{m}(\tilde{x}'|x) = \delta_{\tilde{x},\tilde{x}'}P^{m}(\tilde{x}|x)$, and (\ref{ndklvald})  results in $W_m = 0$, which yields $\mathrm{Cov}(\tilde{Y}) = \psi\mathrm{Cov}(Y)\psi^{\dagger}$. Linear transformations $\phi^m:\mathcal{V}_m\rightarrow \tilde{\mathcal{V}}_m$ all result in mappings $\tilde{Y}_m = \phi^m(Y_m)$ that belong to this deterministic special case (presuming that the maps $\phi^m$ themselves are not random variables) where we let $P^{m}(\tilde{x}|x') = \delta_{\tilde{x},x'}$ and $\tilde{y}^m_{\tilde{x}} = \phi^{m}(y^m_{\tilde{x}})$, thus leading to $\mathrm{Cov}(\phi Y) = \phi\mathrm{Cov}(Y)\phi^{\dagger}$ (cf.~the isomorphisms in  (\ref{nlkalkn}), or the maps $\phi$ used in section \ref{SecUniversalFeatureMaps}).

\section{\label{SecMonotoneFamily} A monotone family of distributions}
Here we shall consider a specific family of multi-partite distributions that is monotone in the sense of the previous section, for which the analysis of the semidefinite decomposition simplifies. We shall in particular consider the case of the triangular scenario in figure \ref{FigTriangle}, which turns out to be convenient for the comparison with the entropic tests, which we consider in section \ref{SecComparison}.
 
\subsection{\label{SecDefiningFamily}Defining the family}
Suppose that we have a collection of variables, each of which has $D\geq 2$ possible outcomes.
In equation (\ref{fbdjkfbdk}) 
we described local operations transforming an initial distribution $P^{M}$. For the local operations we do in this case choose
\begin{equation}
\label{localtransf}
P_p(\tilde{x}|x) := (1-p)\delta_{\tilde{x},x} + p\frac{1}{D}.
\end{equation}
Hence, on each variable we (independently) apply the same type of process, where with probability $p$ we replace the input with a uniformly distributed output, and with probability $1-p$ leave the input intact. Here we choose the input distribution to be $ P^{M}(x_1,\ldots,x_M) = \delta_{x_1,\ldots,x_M}/D$, where the generalized Kronecker delta  is such that  $\delta_{x_1,\ldots,x_M} = 1$ if $x_1 = \cdots =x_M$, while zero otherwise. Hence, $P^{M}(x_1,\ldots,x_M) $ describes $M$ perfectly correlated variables.
By applying  (\ref{fbdjkfbdk})  with the local operations (\ref{localtransf}) we thus obtain a new global distribution 
\begin{equation}
\label{DefPMNp}
\tilde{P}^{M:D}_p(\tilde{x}_1,\ldots,\tilde{x}_M) := \frac{1}{D}\sum_{x_1,\ldots,x_M}P_p(\tilde{x}_1|x_1)\cdots P_p(\tilde{x}_M|x_M)\delta_{x_1,\ldots,x_M},
\end{equation}
where we have added the extra superscript $D$ to indicate the alphabet size of the local random variables.
 By construction, this distribution is permutation symmetric over all the variables.
Moreover, one can confirm that all mono-, bi-, and higher-partite margins of $\tilde{P}^{M:D}_p$ are independent of how many parties $M$ the total distribution $\tilde{P}^{M:D}_p$ involves. For example, the bipartite margin of $\tilde{P}^{M:D}_p$ is equal to $\tilde{P}^{2:D}_p$. Generally, for $M' < M$ it is the case that 
\begin{equation}
\label{MarginalSelfSimilarity}
\tilde{P}^{M':D}_p(\tilde{x}_1,\ldots,\tilde{x}_{M'}) = \sum_{\tilde{x}_{M'+1},\ldots,\tilde{x}_{M}}\tilde{P}^{M:D}_p(\tilde{x}_1,\ldots,\tilde{x}_M).
\end{equation}
Hence, every margin of every family member is another family member.

Since $\tilde{P}_{1}^{M:D}$ is a product distribution over all the observable variables, it is compatible with every bipartite DAG, while $\tilde{P}_{0}^{M:D}$ is perfectly correlated, and thus would only be compatible with bipartite DAGs where some latent variable has edges to all observable variables. 
One can note that the local operations in (\ref{localtransf}) are such that if $1\geq p'\geq p\geq 0$, then there exists a $1\geq q\geq 0$ such that
\begin{equation}
\label{semigroup}
P_{p'}(\tilde{x}|x) = \sum_{x'}P_{q}(\tilde{x}|x')P_{p}(x'|x).
\end{equation}
(Any $1\geq q\geq 0$ is a valid choice if $p=1$, while $q = (p'-p)/(1-p)$ if $1> p\geq 0$.)
Consequently, if $p'\geq p$, then $\tilde{P}^{M:D}_{p'}$ can be generated from $\tilde{P}^{M:D}_p$ by local operations. 
By the reasoning in section \ref{SecMonotonicity} it thus follows that there is some value $p^{*}$ where $\tilde{P}_{p}^{M:D}$ switches from being incompatible to being compatible with the given bipartite DAG (and it cannot switch back again for higher values of $p$). From section \ref{SecMonotonicity} we also know that the semidefinite test also has this monotonic behavior if we choose universal feature maps, although the switch may occur at a lower value of $p$.

\subsection{\label{SecWithinFamily}Within the family: the existence of a semidefinite decomposition is independent of the local alphabet size}
Here we show that the semidefinite test takes a particularly simple form for the family $\tilde{P}^{M:D}_{p}$. In essence we show that the test can be reduced to a test on an $M\times M$ matrix that only depends on $p$, but not on the local alphabet size $D$. A similar result was obtained in (section 4.5 of) \cite{VonPrillwitz15MasterThesis}, for the operator inequalities described in section \ref{SecOperatorInequalities}, but for distributions of the type $v\delta_{x_1,\ldots,x_M}/D-(1-v)/D^2$, while we here consider the family $\tilde{P}_{p}^{M:D}$ defined by (\ref{DefPMNp}).

Suppose that we have an $M$-partite distribution $\tilde{P}^{M:D}_{p}$. We know from the previous section that this distribution is permutation symmetric, and in particular we know from (\ref{MarginalSelfSimilarity}) that all bipartite marginal distributions are of the form $\tilde{P}^{2:D}_{p}$, and all mono-partite marginals are of the form $\tilde{P}^{1:D}_p$. One can moreover  confirm that
\begin{equation}
\label{BinaryAndMono}
\tilde{P}^{2:D}_{p}(\tilde{x},\tilde{x}') =  (1-p)^2\frac{1}{D}\delta_{\tilde{x},\tilde{x}'}+ p(2-p)\frac{1}{D^2},\quad \tilde{P}^{1:D}_p(\tilde{x}_1) = \frac{1}{D}.
\end{equation}

In order to construct a covariance matrix, we here assume feature maps $Y_1,\ldots, Y_M$ that have orthonormal components (i.e.,  feature map $Y_m$ maps the set of possible outcomes of the $m$th random variable to an orthonormal basis of $\mathcal{V}_m$, where $\dim(\mathcal{V}_m) = D$). Hence, the total space $\mathcal{V} = \mathcal{V}_1\oplus\cdots\oplus\mathcal{V}_M$ is $DM$-dimensional, and we can write it as a tensor product $\mathcal{V} = \mathcal{V}^D\otimes \mathcal{V}^M$ of a $D$-dimensional space $\mathcal{V}^D$ and an $M$-dimensional space $\mathcal{V}^M$. By choosing an orthonormal basis $\{e_m\}_{m=1}^{M}$ of $\mathcal{V}^M$, we can identify $\mathcal{V}_m = \mathcal{V}^D\otimes\Sp\{e_m\}$. In section \ref{SecObsLat} we defined the projectors $P_m$ onto the subspaces $\mathcal{V}_m$, and we can write these projectors as
\begin{equation}
\label{anvldfadfn}
P_m = \hat{1}_D\otimes \hat{e}_m,
\end{equation} 
where $\hat{1}_D$ is the identity operator on $\mathcal{V}^D$, and $\hat{e}_m$ is the projector onto $e_m$.

The covariance matrix $\mathrm{Cov}(Y)$ for the random variable $Y = Y_1 +\cdots +Y_M$ is an $MD\times MD$ matrix and takes a particularly simple form
\begin{equation}
\label{nadjkfba}
\mathrm{Cov}(Y) = \left[\begin{matrix}
         Q & (1-p)^2Q & \cdots & (1-p)^2Q\\
 (1-p)^2Q    & Q & \ddots &  \vdots\\
\vdots                               & \ddots    & Q & (1-p)^2 Q \\
(1-p)^2Q & \cdots    & (1-p)^2Q &    Q
\end{matrix}\right] = \frac{1}{D}Q\otimes C(p), 
\end{equation}
where we define the $M\times M$ matrix
\begin{equation}
\label{vnoekfnbf}
C(p) := \left[\begin{matrix} 
1 & (1-p)^2 & \cdots & (1-p)^2 \\
(1-p)^2 & 1 & \ddots & \vdots\\
\vdots     & \ddots &  & (1-p)^2\\
(1-p)^2   & \cdots & (1-p)^2 & 1
\end{matrix}\right]
\end{equation}
and the $D\times D$ matrix $Q$ with elements
\begin{equation}
Q_{\tilde{x},\tilde{x}'}  
:=  \delta_{\tilde{x},\tilde{x}'} -\frac{1}{D},\quad \tilde{x},\tilde{x}' = 1,\ldots, D.
\end{equation}
Note that we can write $Q = \hat{1}_D - cc^{\dagger}$, where $c = (1,\ldots,1)^{\dagger}/\sqrt{D}\in\mathcal{V}^{D}$ is normalized. Hence, $Q$ is the projector onto the $(D-1)$-dimensional subspace of $\mathcal{V}^D$ that is the orthogonal complement to the one-dimensional subspace spanned by $c$. From $Q$ being  a projector, it also follows that $Q\geq 0$.

Suppose now that we have a particular bipartite DAG $B$ with observable variables $O_1,\ldots,O_M$ and latent variables $L_1,\ldots, L_N$. As we recall from section \ref{SecDecBipartDAGs}, the semidefinite test is characterized via the projectors $P^{(n)} = \sum_{m\in \mathrm{ch}(L_n)}P_m$ as
\begin{equation}
\label{djfvna}
\mathrm{Cov}(Y) = R + \sum_{n=1}^NC_n,\quad P^{(n)}C_{n}P^{(n)} = C_n,\quad C_n\geq 0, \quad \sum_{m=1}^{M}P_mRP_m = R,\quad R\geq 0.
\end{equation}
In the present case, we can write these projectors as
\begin{equation}
\label{aklndfkal}
 P^{(n)} = I_D\otimes \tilde{P}^{(n)},\quad \tilde{P}^{(n)} = \sum_{m\in \mathrm{ch}(L_n)}\tilde{P}_m,
\end{equation}
with $\tilde{P}_m$ as in (\ref{anvldfadfn}).

For each fixed number of observable variables $M$, local alphabet size $D$, and given bipartite DAG $B$,  we know that the family $\tilde{P}^{M:D}_p$ is monotone with respect to $p$, in the sense that the covariance matrix $\mathrm{Cov}(Y)$ satisfies the semidefinite decomposition for all $p$ beyond a certain threshold value, while it is violated for all values below. The following proposition shows that this threshold is independent of $D$, and that it can be determined via simplified decomposition of the matrix $C(p)$.

\begin{Proposition}
\label{PropCovDec}
Let $\mathrm{Cov}(Y)$ be the covariance matrix, for feature maps with orthonormal components, corresponding to the distribution $\tilde{P}^{M:D}_p$, as defined in (\ref{DefPMNp}), for $M$ observable variables,  and local alphabet size $D\geq 2$. 
For each value $1\geq p \geq 0$ it is the case that $\mathrm{Cov}(Y)$ satisfies the semidefinite decomposition (\ref{djfvna}) with respect to a given bipartite DAG $B$, if and only if $C(p)$, defined in (\ref{vnoekfnbf}), satisfies the decomposition 
\begin{equation}
\label{mjmh}
C(p) = \tilde{R} + \sum_{n=1}^{N}\tilde{C}_n,\quad \tilde{P}^{(n)}C_n\tilde{P}^{(n)},\quad \tilde{C}_n\geq 0,\quad \sum_{m=1}^{N}\tilde{P}_m\tilde{R}\tilde{P}_m = \tilde{R}, \quad \tilde{R}\geq 0. 
\end{equation}
Moreover, there exists a number $1\geq \overline{p}(B) \geq 0$ that does not depend on $D$, such that  $\mathrm{Cov}(Y)$ satisfies (\ref{djfvna}) and $C(p)$ satisfies (\ref{mjmh}) for all $p >\overline{p}(B)$, while $\mathrm{Cov}(Y)$ and $C(p)$ do not satisfy the decompositions for $p <\overline{p}(B)$.
\end{Proposition}

\begin{proof}
First we shall show that if $C(p)$ satisfies the decomposition, then $\mathrm{Cov}(Y)$ also satisfies the decomposition.
Let $p$ be any $1\geq p\geq 0$ such that there exists a semidefinite decomposition of $C(p)$ as in (\ref{vnoekfnbf}).
Equation  (\ref{mjmh}) provides $\tilde{R}$ and $\tilde{C}_n$. Define $R := Q\otimes\tilde{R}/D$ and $C_n := Q\otimes \tilde{C}_n/D$. Thus defined, it follows that
\begin{equation*}
R + \sum_{n}C_n = \frac{1}{D}Q\otimes (\tilde{R} + \sum_{n}\tilde{C}_n) = \frac{1}{D}Q\otimes C(p) = \textrm{Cov}(Y). 
\end{equation*}
Moreover, by the conditions in (\ref{mjmh}) and the observations in (\ref{aklndfkal}), it follows that 
\begin{equation*}
P^{(n)}C_nP^{(n)} = [\hat{1}_D\otimes \tilde{P}^{(n)}][\frac{1}{D}Q\otimes \tilde{C}_n] [\hat{1}_D\otimes \tilde{P}^{(n)}]  = C_n,
\end{equation*}
 and $C_n = Q\otimes \tilde{C}_n/D\geq 0$.

 Furthermore, by  the conditions in (\ref{mjmh})  and (\ref{anvldfadfn}), it follows that  
\begin{equation*}
\sum_m P_mR P_m = \sum_{m}[\hat{1}_D\otimes \tilde{P}_m][\frac{1}{D}Q\otimes\tilde{R}][\hat{1}_D\otimes \tilde{P}_m] = R,
\end{equation*}
 and $R = Q\otimes\tilde{R}/D\geq 0$. Hence, this procedure produces a valid semidefinite decomposition of $\mathrm{Cov}(Y)$. 
Hence, for every $p$ for which $C(p)$ has a valid decomposition, it follows that $\mathrm{Cov}(Y)$ also has a valid decomposition. 

Next we prove the opposite implication, namely that the existence of a decomposition of $\mathrm{Cov}(Y)$ implies a decomposition  of $C(p)$.  Let us thus assume that there is a $1\geq p\geq 0$ for which there exists a decomposition of $\mathrm{Cov}(Y)$ as in (\ref{djfvna}). Equation  (\ref{djfvna}) provides $R$ and $C_n$. Let $v\in\mathcal{V}^D$ be normalized, and such that $Qv = v$. Such a $v$ always exists, since $Q$ is a projector onto a $(D-1)$-dimensional subspace of $\mathcal{V}^D$ and $D\geq 2$. Define
$\tilde{R}:=Dv^{\dagger}Rv$ and $\tilde{C}_n := Dv^{\dagger}C_nv$ (where one should keep in mind that e.g.~$v^{\dagger}Rv$ is an operator on $\mathcal{V}^M$, since $v\in\mathcal{V}^D$). Hence, by (\ref{djfvna}) and (\ref{nadjkfba})
\begin{equation*}
\begin{split}
\tilde{R} +\sum_n\tilde{C}_n = & Dv^{\dagger}(R +\sum_nC_n)v =  Dv^{\dagger}\mathrm{Cov}(Y)v=  Dv^{\dagger}[\frac{1}{D}Q\otimes C(p)]v = v^{\dagger}Qv C(p) = C(p).
\end{split}
\end{equation*}
Moreover, by the conditions in (\ref{djfvna}) and the observations in (\ref{aklndfkal}), it follows that 
\begin{equation*} 
\begin{split}
\tilde{P}^{(n)}\tilde{C}_n\tilde{P}^{(n)} = & \tilde{P}^{(n)}  Dv^{\dagger}C_nv \tilde{P}^{(n)} =  Dv^{\dagger}[\hat{1}_D\otimes\tilde{P}^{(n)}]   C_n[\hat{1}_D\otimes\tilde{P}^{(n)}]v  = Dv^{\dagger}P^{(n)}   C_n P^{(n)}v =   Dv^{\dagger} C_nv = \tilde{C}_n,
\end{split}
\end{equation*}
and $\tilde{C}_n := Dv^{\dagger}C_nv\geq 0$.
Furthermore, (\ref{djfvna}) and  (\ref{anvldfadfn}) yields
\begin{equation*} 
\begin{split}
\sum_m\tilde{P}_m\tilde{R}\tilde{P}M = & \sum_m \tilde{P}_m Dv^{\dagger}Rv \tilde{P}_m = \sum_m Dv^{\dagger}[\hat{1}_D\otimes\tilde{P}_{m}]   R[\hat{1}_D\otimes\tilde{P}_m]v =   Dv^{\dagger}(\sum_m P_{m}R P_m)v  =  Dv^{\dagger}Rv =  \tilde{R},
\end{split}
\end{equation*}
and $\tilde{R}:=Dv^{\dagger}Rv\geq 0$. Hence, we can conclude that the decomposition of $\mathrm{Cov}(Y)$ induces a valid decomposition of $C(p)$ as in (\ref{mjmh}).

We know from section \ref{SecDefiningFamily} that the family $\tilde{P}^{M:D}_p$ is monotone, in the sense that $\mathrm{Cov}(Y)$ (since it is based on orthonormal feature maps) satisfies the semidefinite decomposition for all $p$ beyond a certain threshold value, which we can call  $\overline{p}(B)$, while violating the decomposition for all $p$ below $\overline{p}(B)$. From the above equivalence we conclude that the same transition is valid for $C(p)$ with respect to the decomposition in (\ref{mjmh}).
\end{proof}

\subsection{\label{SecTripartiteMonotone} Compatibility with the triangular DAG}
Here we consider the tripartite case  and determine the value of $p$ where $\tilde{P}^{3:D}_p$ switches from not satisfying the semidefinite decomposition, to satisfying it, with respect to the  triangular scenario in figure \ref{FigTriangle}.
The family of distributions $\tilde{P}^{M:D}_p$, defined in (\ref{DefPMNp}), does in the tripartite case take the form
\begin{equation}
\label{ndvklanlkv}
\begin{split}
\tilde{P}^{3:D}_p(\tilde{x}_1,\tilde{x}_2,\tilde{x}_3) 
=  &  (1-p)^3\frac{1}{D}\delta_{\tilde{x}_1,\tilde{x}_2,\tilde{x}_3}\\
& + p(1-p)^2\frac{1}{D^2}[\delta_{\tilde{x}_1,\tilde{x}_2}  +\delta_{\tilde{x}_1,\tilde{x}_3} + \delta_{\tilde{x}_2,\tilde{x}_3}]\\
& + p^2(3-2p)\frac{1}{D^3},
\end{split}
\end{equation}
and the matrix $C(p)$ and the projectors $ \tilde{P}^{(1)}$, $ \tilde{P}^{(2)}$, and  $\tilde{P}^{(3)}$ become
\begin{equation*}
C(p) = \left[\begin{matrix}
1 & (1-p)^2 & (1-p)^2 \\
(1-p)^2 & 1 & (1-p)^2 \\
(1-p)^2 & (1-p)^2 & 1
\end{matrix}\right],
\quad \tilde{P}^{(1)} = \left[\begin{matrix} 0 & 0 & 0\\
0 & 1 & 0 \\
0 & 0 & 1 
\end{matrix}\right], \quad \tilde{P}^{(2)} = \left[\begin{matrix} 1 & 0 & 0\\
0 & 0  & 0 \\
0 & 0 & 1 
\end{matrix}\right],\quad \tilde{P}^{(3)} = \left[\begin{matrix} 1 & 0 & 0\\
0 & 1 & 0 \\
0 & 0 & 0
\end{matrix}\right].
\end{equation*}

As a corollary of Proposition \ref{PropCovDec} we here determine the `transition point' $\overline{p}(B)$ for the family $\tilde{P}^{M:D}_p$ in the triangular scenario.

\begin{Lemma}
\label{bvdsaajkvb}
For $p\in \mathbb{R}$ it is the case that 
$\left[\begin{smallmatrix}
\frac{1}{2} & (1-p)^2\\
(1-p)^2 & \frac{1}{2}
\end{smallmatrix}\right]\geq 0$ $\Leftrightarrow$ $1 - \frac{1}{\sqrt{2}}\leq p \leq 1+\frac{1}{\sqrt{2}}$.
\end{Lemma}

\begin{Lemma}
\label{nklknbj}
Let $a,b,r\in \mathbb{C}$, then 
$\left[\begin{smallmatrix}
a & r \\
r & b
\end{smallmatrix}\right]\geq 0$ $\Leftrightarrow$ $\left[\begin{smallmatrix}
b & r \\
r & a
\end{smallmatrix}\right]\geq 0$.
\end{Lemma}

\begin{Corollary}
\label{PropTransition}
For the family $\tilde{P}^{3:D}_p$ in equation (\ref{ndvklanlkv}), and for feature maps with orthonormal components, the covariance matrix $\mathrm{Cov}(Y)$ has a semidefinite decomposition with respect to the triangular bipartite DAG $B$ in figure \ref{FigTriangle}, if and only if $1-\frac{1}{\sqrt{2}}\leq p\leq 1$. Hence, $\overline{p}(B) = 1-1/\sqrt{2}$.
\end{Corollary}
One may note that $\tilde{P}^{3:D}_p$ has a semidefinite decomposition also in the case $p = 1-1/\sqrt{2}$, i.e., at the transition point. Proposition \ref{PropCovDec} does strictly speaking leave open the nature of the transition point \emph{per se}.
\begin{proof}
By Proposition \ref{PropCovDec} we know that it is sufficient to determine the $p$ for which $C(p)$ decomposes as in (\ref{mjmh}).
Due to Lemma \ref{bvdsaajkvb} it follows that 
\begin{equation*}
\tilde{R} = 0,\,\, \tilde{C}_1 = \left[\begin{smallmatrix}
0 & 0 &0\\
0 & \frac{1}{2} & (1-p)^2\\
0 & (1-p)^2 & \frac{1}{2}
\end{smallmatrix}\right],\,\,  \tilde{C}_2 = \left[\begin{smallmatrix}
 \frac{1}{2} & 0 & (1-p)^2\\
0 & 0 & 0\\
 (1-p)^2 & 0 & \frac{1}{2}
\end{smallmatrix}\right],\,\,  \tilde{C}_3 = \left[\begin{smallmatrix}
 \frac{1}{2} &  (1-p)^2 & 0\\
 (1-p)^2 & \frac{1}{2} & 0\\
0 & 0 & 0
\end{smallmatrix}\right]
\end{equation*}
satisfy the decomposition (\ref{mjmh}) for all $1 - 1/\sqrt{2}\leq p \leq 1$.
However, this does not exclude the possibility that there exists some other decomposition that yields a smaller $p$.

Suppose that $0 \leq p'< 1-1/\sqrt{2}$.
By the structure of the triangular DAG, it follows that the most general decomposition of the form (\ref{mjmh}) possible (incorporating the diagonal matrix $\tilde{R}$ into $\tilde{C}_1$, $\tilde{C}_2$, and $\tilde{C}_3$) can be written $C(p) = \tilde{C}_1 +\tilde{C}_2 +\tilde{C}_3$, where
\begin{equation*}
\tilde{C}_1 = \left[\begin{smallmatrix}
0 & 0 &0\\
0 & b_2 & (1-p')^2\\
0 & (1-p')^2 & c_1
\end{smallmatrix}\right],\,\,  \tilde{C}_2 = \left[\begin{smallmatrix}
 a_1 & 0 & (1-p')^2\\
0 & 0 & 0\\
 (1-p')^2 & 0 & c_2
\end{smallmatrix}\right],\,\,  \tilde{C}_3 = \left[\begin{smallmatrix}
 a_2 &  (1-p')^2 & 0\\
 (1-p')^2 & b_1 & 0\\
0 & 0 & 0
\end{smallmatrix}\right],
\end{equation*}
and where $a_1,a_2,b_1,b_2,c_1,c_2\geq 0$ and $a_1 +a_2 = 1$, $b_1 +b_2 = 1$, $c_1 + c_2 = 1$. 
By the assumed semidefiniteness of $\tilde{C}_1$, $\tilde{C}_2$, and $\tilde{C}_3$, it follows that 
\begin{equation}
\label{cghcfgx}
\begin{split}
 M_1:= \left[\begin{smallmatrix}
a_1 & (1-p')^2 \\
(1-p')^2 & c_2 
\end{smallmatrix}\right] \geq 0,\quad M_2:= \left[\begin{smallmatrix}
a_2 & (1-p')^2 \\
 (1-p')^2  & b_1\\
\end{smallmatrix}\right]\geq 0,\quad  M_3 :=\left[\begin{smallmatrix}
 b_2 & (1-p')^2 \\
 (1-p')^2  & c_1\\
\end{smallmatrix}\right]\geq 0.
\end{split}
\end{equation}
 By Lemma \ref{nklknbj} it follows that (\ref{cghcfgx}) implies
\begin{equation*}
\begin{split}
 M_4:= \left[\begin{smallmatrix}
c_2 & (1-p')^2 \\
(1-p')^2 & a_1 
\end{smallmatrix}\right] \geq 0,\quad M_5:= \left[\begin{smallmatrix}
b_1 & (1-p')^2 \\
 (1-p')^2  & a_2\\
\end{smallmatrix}\right]\geq 0,\quad M_6 :=\left[\begin{smallmatrix}
 c_1 & (1-p')^2 \\
 (1-p')^2  & b_2\
\end{smallmatrix}\right]\geq 0.
\end{split}
\end{equation*}
Since these matrices all are positive semidefinite, it follows that every convex combinations of them is also positive semidefinite. 
Thus one can confirm that
\begin{equation}
\begin{split}
\left[\begin{matrix}
\frac{1}{2} & (1-p')^2 \\
(1-p')^2 & \frac{1}{2}
\end{matrix}\right]  = \frac{1}{6}M_1  + \frac{1}{6}M_2 + \frac{1}{6}M_3 +\frac{1}{6}M_4 +  \frac{1}{6}M_5 +\frac{1}{6}M_6  \geq 0.
\end{split}
\end{equation}
However, the positive semidefiniteness of this matrix is a contradiction to Lemma \ref{bvdsaajkvb}, since by assumption $p'< 1-1/\sqrt{2}$. Hence, $C(p)$ can only have a decomposition as in (\ref{mjmh}) if $1-1/\sqrt{2}\leq p\leq 1$. By Proposition \ref{PropCovDec} it thus follows that $\mathrm{Cov}(Y)$ satisfies the semidefinite decomposition as in (\ref{djfvna}) if and only if $1-1/\sqrt{2}\leq p\leq 1$.

\end{proof}

\section{\label{SecComparison} Comparison with entropic tests}
Outer relaxations of the compatibility set corresponding to latent variable structures, based on information theoretic inequalities, have been considered previously \cite{steudel2015information,Chaves2014,Chaves2014b,weilenmann2016non}. Here we make a numerical comparison of the performance of these entropic tests and the semidefinite test. A basic challenge is that we in practice do not know the true set of compatible distributions. However, since we are dealing with outer approximations, a reasonable approach is to compare how `strict' the tests are, i.e., if one test generally tends to reject more distributions than the other. 

Given the rather radical difference in appearance and functional form between the semidefinite test and tests based on entropy inequalities (described in more detail in the next section) it is far from clear how these tests relate, or if there even is a clear-cut relation in the sense that one would be systematically stronger than the other. An indication can be gained from \cite{VonPrillwitz15MasterThesis}, where it was found that tests based on operator inequalities, of the type described in section \ref{SecOperatorInequalities}, appear to be stronger than the entropic ones for small alphabet sizes, but that there seems to be a switchover for larger alphabets (see section 4.5 of \cite{VonPrillwitz15MasterThesis}). Here we confirm similar trends for the semidefinite test in comparison with the entropic test, where we focus on the `triangular' DAG described in figure \ref{FigTriangle}. 
In case of binary variables, we do in section \ref{SecRandomIsing} make a comparison over an ensemble of randomly constructed distributions. However, our major testbed for these comparisons (in section \ref{SecComparisonMonotone}) is the family of distributions $\tilde{P}^{M:D}_p$ introduced in section \ref{SecMonotoneFamily}.

\subsection{Entropy inequalities for the triangular DAG}
We focus on the triangular DAG in figure \ref{FigTriangle}, since this has been a rather well investigated scenario with several known entropic inequalities associated with it. For the three observable variables $O_1,O_2,O_3$  we let $H(1) := H(O_1) := -\sum_j P(O_1 = j)\log_{2}P(O_1 = j)$ denote the Shannon entropy, and in a similar manner $H(12) := H(O_1,O_2)$, etc, where `$\log_2$' denotes the base $2$ logarithm. The first inequality  (\ref{ineqE1}) for the triangular scenario was obtained  in \cite{Fritz2012} (see also \cite{Chaves2014} and \cite{weilenmann2016non})
\begin{equation}
\label{ineqE1} E_1  :=  - H(1)-H(2)-H(3)+H(13) +H(12) \geq 0.
\end{equation}
The following two inequalities were derived in \cite{Chaves2014}
\begin{equation}
\begin{split}
\label{ineqE2}
 E_2  := & -3H(1) -3H(2) -3H(3) + 2H(12) + 2H(13) + 3H(23) -H(123)\geq 0,
 \end{split}
\end{equation}
\begin{equation}
\begin{split}
\label{ineqE3}  E_3  := & -5H(1) - 5H(2) -5H(3)  + 4H(12) + 4H(13) + 4H(23) -2H(123)\geq 0.
 \end{split}
\end{equation}
Finally, inequalities (\ref{ineqE4}) to (\ref{ineqE6}) were obtained in \cite{weilenmann2016non} 
\begin{equation}
\begin{split}
 \label{ineqE4} E_4  := & -4H(1) -4H(2)-4H(3) + 3H(12) + 3H(13) + 4H(23) -2H(123)\geq 0,\\
 \end{split}
\end{equation}
\begin{equation}
\begin{split}
\label{ineqE5} E_5  := & -2H(1)-2H(2)-2H(3) + 3H(12) + 3H(13) + 3H(23) -4H(123)\geq 0,\\
 \end{split}
\end{equation}
\begin{equation}
\begin{split}
\label{ineqE6} E_6  := & -8H(1)-8H(2)-8H(3)  +7H(12) +7H(13) + 7H(23)-5H(123) \geq 0. 
 \end{split}
\end{equation}
One should observe that the expressions in (\ref{ineqE1}), (\ref{ineqE2}), and (\ref{ineqE4}) are not symmetric under permutations of the $O_1,O_2,O_3$, and  thus each of these generate two more inequalities. Whenever one of these inequalities is violated we can conclude that the observable distribution cannot originate from the bipartite DAG in figure \ref{FigTriangle}.

One may note that all of these entropic inequalities, apart from  (\ref{ineqE1}), depend on the full tripartite distribution, while the semidefinite test only takes into account the mono- and bipartite marginals.  One may thus intuitively suspect that the semidefinite test would be at a disadvantage compared to these tripartite entropic tests.

\subsection{\label{SecRandomIsing}Rejection rates in random Ising models: The binary case}
For a numerical comparison between the entropic and the semidefinite test for the triangular scenario in figure \ref{FigTriangle}, we assume binary variables $O_1,O_2,O_3\in\{-1,1\}$, and distributions $P(\overline{x}) := P(O_1 =x_1,O_2=x_2,O_3 = x_3)$, $\overline{x} := (x_1,x_2,x_3)$ given by an Ising interaction model \cite{GallavottiStatisticalMechanics,KollerProbabilisticGraphicalModels}
\begin{equation}
\label{IsingModel}
P(\overline{x})  = \frac{e^{- \overline{x}^{\dagger}J\overline{x}}}{Z},
\end{equation}
with $Z$ being the normalization constant, and where $J$ is a real $3\times 3$ matrix. For each single instance of this model we draw the elements of $J$ independently from a Gaussian distribution with zero mean  and variance $1$. 

For the semidefinite test we choose (universal) feature maps that associate the outcomes of the random variables to elements of orthonormal bases, thus resulting in a $6\times 6$ covariance matrix. The semidefinite test was implemented via a semidefinite program that minimizes a constant function, thus effectively testing whether there exist any feasible elements.

 For each instance over  $10^6$ independent repetitions of the Ising model in (\ref{IsingModel}) we performed the semidefinite test, as well as tested the entropic inequalities (\ref{ineqE1}) to (\ref{ineqE6})  together with all their permutations.

 The following table gives the approximate fraction of rejections. In the table,  $E^{\cup}_1$ (and analogously for $E^{\cup}_2$ and $E^{\cup}_4$) means that  we test the inequality in (\ref{ineqE1}) as well as its two permutations, and we count the fraction of the sample that violates any of these three inequalities, i.e., we take the union of the corresponding rejection regions. The entry `Combined' signifies the fraction of rejections due to violations of at least one of the inequalities (\ref{ineqE1}) to (\ref{ineqE6}) or any of their permutations. Finally `Semidefinite'  denotes the fraction of rejections for the semidefinite test.

\begin{equation*}
\begin{split}
& E_1^{\cup}:   0.57,\quad  E_2^{\cup}:  0.60, \quad E_3: 0.54, \\
& E_4^{\cup}:   0.63,\quad   E_5:  0.40, \quad E_6: 0.60,\\
& \textrm{Combined}:  0.64,\\
& \textrm{Semidefinite}:  0.77
\end{split}
\end{equation*}

Since the fraction of rejections is higher for the semidefinite test than for all the entropic inequalities combined, this suggests that the semidefinite test in some sense has a   `larger' region of rejection, and thus would be the stronger test. To get some information on the relation between the two regions of rejections, we checked whether we could find any case where the semidefinite test accepted an instance that had been rejected by some of the entropic inequalities. However, we could find no such case, which suggests that the region of rejection for the collection of entropic inequalities is contained in the region of  rejection for the semidefinite test.

\subsection{\label{SecComparisonMonotone}Comparison on a monotone family of distributions}

In section \ref{SecMonotonicity} we argued that the compatibility of distributions with respect to a given bipartite DAG is monotonous under local operations, and that the semidefinite test  also satisfies this property if we use universal feature maps.
In section \ref{SecTripartiteMonotone} we introduced a particular tripartite family of distributions  $\tilde{P}^{3:D}_{p}$ that can be generated from the appropriate maximally correlated distribution by local operations, and where we could show that this family cut the boundary of the semidefinite compatibility region at $p = 1-1/\sqrt{2}$. Here we compare the performance of the entropic tests with the semidefinite test on this particular family of distributions. 

\subsubsection{\label{SecBinary}Binary variables}

We begin in the case of three binary variables, i.e., each variable can take two possible values. In this case $\tilde{P}^{3:2}_{p}$ reduces to 
\begin{equation}
\label{tildeP3} 
\begin{split}
\tilde{P}^{3:2}_p(\tilde{x}_1,\tilde{x}_2,\tilde{x}_3) = \left\{\begin{matrix} 
\frac{1}{8}(4-6p+3p^2),\quad \textrm{if}\quad \tilde{x}_1 = \tilde{x}_2 = \tilde{x}_2,\\
\frac{1}{8}p(2-p),\quad \textrm{otherwise}.
\end{matrix}\right.
\end{split}
\end{equation}

\begin{figure}[h!]
 \includegraphics[width= 9cm]{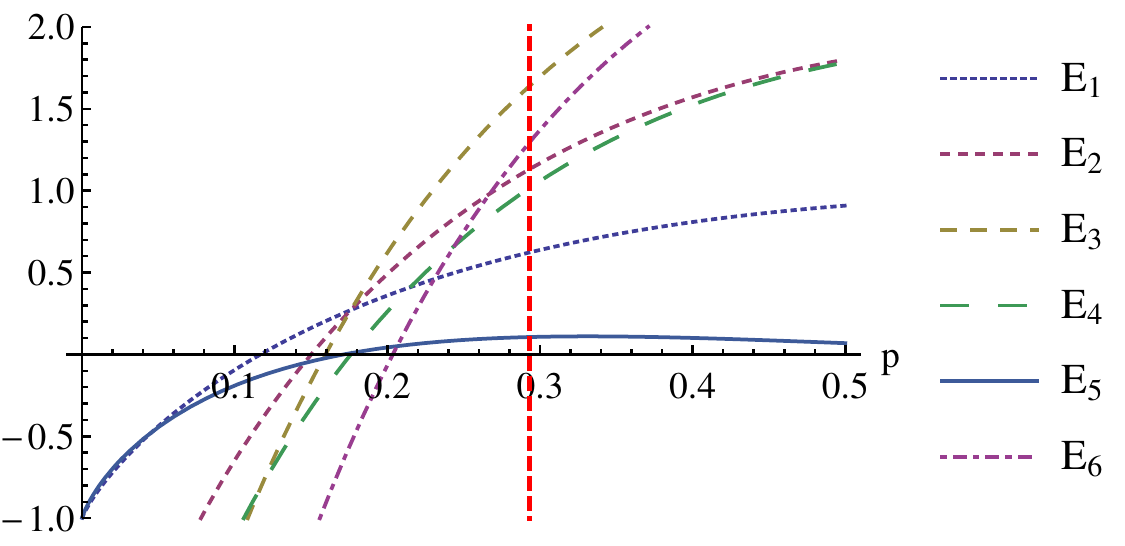} 
\caption{\label{FigEntropyBinary} {\bf Entropic versus semidefinite for binary variables.} For three binary variables described by the distribution $\tilde{P}^{3:2}_{p}$ in equation (\ref{tildeP3}), we calculate $E_1,\ldots, E_6$ defined in (\ref{ineqE1}) to (\ref{ineqE6}) as functions of the parameter $p$. When one of these functions turns negative, it implies that the distribution  $\tilde{P}^{3:2}_{p}$ is not compatible with the triangular bipartite DAG in figure \ref{FigTriangle}. Moreover, we determine the $6\times 6$ covariance matrix with respect to feature maps that assign orthogonal vectors to the outcomes. The red vertical line indicates the value $p = 1-1/\sqrt{2} \approx 0.29$, determined in section \ref{SecTripartiteMonotone}, below which the semidefinite test rejects the resulting covariance matrix. As one can see, the semidefinite test has the larger region of rejection, and is in this sense the stronger test for this particular binary setup.
}
\end{figure}

In figure \ref{FigEntropyBinary} we plot $E_1,\ldots,E_6$ as functions of the parameter $p$. The entropic test rejects the model for a given $p$ whenever one of these functions become negative.
 For the calculation of the covariance matrix we choose feature maps that assign orthonormal vectors to the outcomes of the three random variables, thus being universal.  As one can see from figure  \ref{FigEntropyBinary}, the semidefinite test starts to reject at higher values of $p$ than all the entropic tests, and is thus closer to the true value $p^{*}$ of the transition than any of the entropic tests.

\subsubsection{\label{SecE1asymptotics}Asymptotics of the $E_1$ test}

 $E_1$ defined in (\ref{ineqE1}) is the only of the entropic quantities (\ref{ineqE1}) to (\ref{ineqE6})  that solely includes mono- and bipartite marginals; the others also depend on the full tripartite distribution. Since the test based on $E_1$ and the semidefinite test thus are on `equal footing' in this regard, it appears relevant to pay some additional attention to the relation between these two tests. In section \ref{SecTripartiteMonotone}, and in particular in Corollary \ref{PropTransition},   we proved that the distribution $\tilde{P}^{3:D}_p$, defined in equation (\ref{ndvklanlkv}), satisfies the semidefinite test if and only if $p\geq 1-1/\sqrt{2}$, irrespective of the alphabet size $D$. Hence, the `transition point' for the semidefinite test is independent of $D$ for this particular family of distributions. Here we shall show that the corresponding transition point for the test based on $E_1$ lies below $1-1/\sqrt{2}$, but asymptotically approaches this value as $D$ increases.

The family of distributions $\tilde{P}^{3:D}_p$ in (\ref{ndvklanlkv}) is permutation symmetric with respect to the three parties, and $E_1$ can, via equation (\ref{BinaryAndMono}), be evaluated as
\begin{equation}
\label{nfajlbakl}
\begin{split}
E_1 = & -3H(1) +2H(12)\\
= &  -3\log D \\
& -2(1-\frac{1}{D}) p(2-p)\log \Big[ p(2-p)\frac{1}{D^2}\Big]\\
& -2\Big[(1-p)^2+ p(2-p)\frac{1}{D}\Big]\log \Big[(1-p)^2\frac{1}{D}+ p(2-p)\frac{1}{D^2}\Big].
\end{split}
\end{equation}

One can  confirm that $E_1(0) = -\log D$, $E_1(1) = \log D$, and 
\begin{equation}
\begin{split}
\frac{dE_1}{dp} =  4(1- \frac{1}{D})(1-p)\log\Big[1 + D\frac{(1-p)^2}{p(2-p)}\Big],
\end{split}
\end{equation}
which is non-negative for $0\leq p\leq 1$. Hence, for each fixed $D$, the function $E_1$ is monotonically increasing  for $0\leq p\leq 1$, and thus the equation $E_1(p) = 0$  has exactly one root, which is situated somewhere in the open interval $(0,1)$. Thus, analogous to the semidefinite test, the test based on $E_1$ will reject all elements in the family $\tilde{P}^{3:D}_p$ below a certain transition point, and accept all distributions above that value.
Next one can confirm that 
\begin{equation*}
E_1(1-\frac{1}{\sqrt{2}}) = \log \frac{2D}{D+1}   +\frac{1}{D} \log\frac{2^D}{ D+1}>0,\quad D = 2,3,\ldots.
\end{equation*}
Since $E_1$ thus is monotonously increasing with respect to $p$, we can conclude that the root $\tilde{p}$ of $E_1(\tilde{p}) = 0$ is such that  $\tilde{p} < 1-1/\sqrt{2}$ for all $D\geq 2$. Finally we wish to the determine the asymptotic value of the root $\tilde{p}$ as $D\rightarrow\infty$. To this end we rewrite (\ref{nfajlbakl}) such that we highlight the different orders of dependency on $D$,
\begin{equation}
\label{adflbnknbl}
\begin{split}
E_1 =&  2\big[1+\frac{1}{\sqrt{2}}-p\big]\big[p-(1 -\frac{1}{\sqrt{2}})\big]  \log D\\
& -2 p(2-p)\log[p(2-p)] -4(1-p)^2\log (1-p)\\
& -2 p(2-p)\frac{1}{D}\log D\\
& -2(1-p)^2\log \Big[1+\frac{p(2-p)}{(1-p)^2}\frac{1}{D}\Big]\\
& +2p(2-p)\frac{1}{D}\log\left[\frac{p (2-p)}{(1-p)^2}\right]\\
& -2p(2-p)\frac{1}{D}\log \Big[1+\frac{p(2-p)}{(1-p)^2}\frac{1}{D}\Big].
\end{split}
\end{equation}
On any interval $\delta \leq p \leq 1-\delta$, with $1/2>\delta >0$, the last four lines of (\ref{adflbnknbl}) each approaches zero as $D\rightarrow\infty$.
Moreover, one can note that the leading order term (in the first line) does for each fixed $D$ increase monotonically for $0 \leq p\leq 1$ and switches from negative to positive at $p = 1-1/\sqrt{2}$. If one fixes  $\epsilon >0$, one can realize that for all sufficiently large $D$, it is the case that $E_1(p)>0$ for all $1-1/\sqrt{2} +\epsilon \leq p\leq 1-\delta$, and $E_1(p) \leq 0$ for all $0\leq p\leq 1-1/\sqrt{2}-\epsilon$. We can thus conclude that the root $\tilde{p}$
 of $E_1(\tilde{p}) =0$ in the interval $0\leq \tilde{p}\leq 1$ approaches $1-1/\sqrt{2}$ as $D\rightarrow \infty$.

\subsubsection{\label{SecIncreasingAlphabets}Comparison on increasing alphabets}

\begin{figure}[h!]
 \includegraphics[width= 15cm]{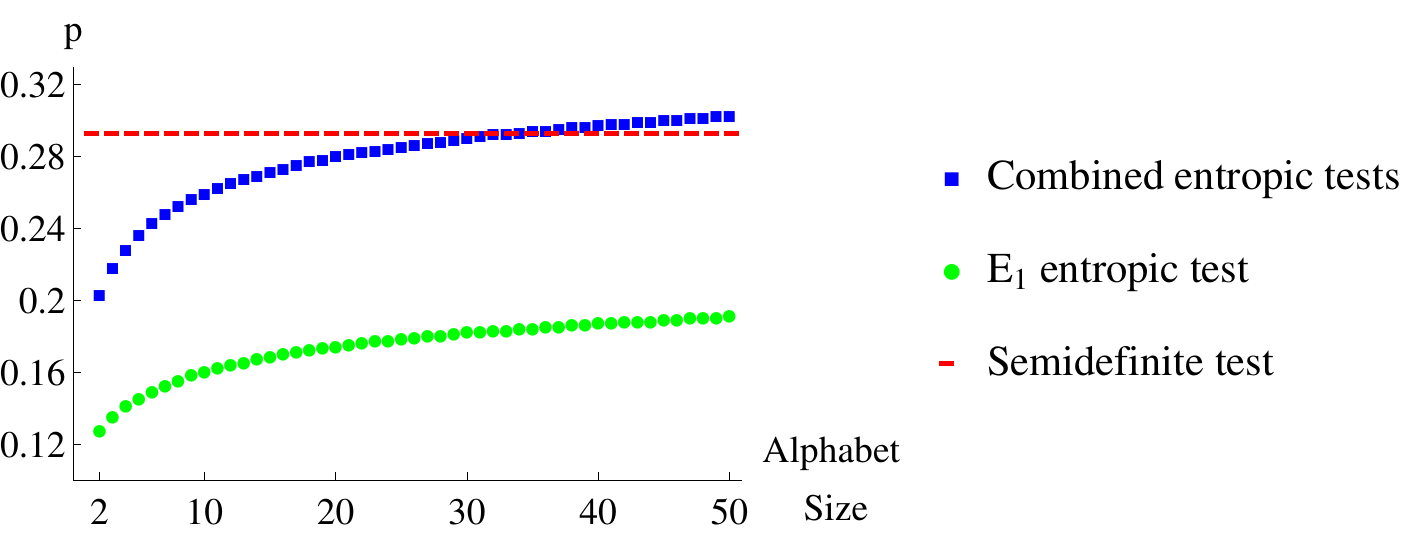} 
\caption{\label{FigAlphabetSize} {\bf Entropic versus semidefinite tests for increasing alphabet sizes.} 
For the distribution $\tilde{P}^{3:D}_{p}$ in (\ref{ndvklanlkv}) we compare the entropic and semidefinite test as functions of $D$.
Here we determine the smallest value of $p$ for which the respective test accepts $\tilde{P}^{3:D}_{p}$, as a function of the local alphabet size $D$. From section \ref{SecTripartiteMonotone} we know that the transition point for the semidefinite test is $p = 1-1/\sqrt{2} \approx 0.29$,  independently of $D$ (the red dashed line). We also plot (blue squares) the minimal value of $p$ for which all of the entropic inequalities (\ref{ineqE1}) to (\ref{ineqE6}) are satisfied, as a function of $D$. The transition point for this entropic test crosses the red line at $D = 32$. Hence, for the class of functions $\tilde{P}^{3:D}_{p}$, the entropic tests becomes stronger than the semidefinite test for alphabet sizes beyond $32$. Finally, we plot (green circles) the minimal value of $p$ for which $E_1(p)\geq 0$, as a function of $D$. By section \ref{SecE1asymptotics} we know that this transition point asymptotically reaches $1-1/\sqrt{2}$.
}
\end{figure}

In the previous section we found that the semidefinite test is stronger than the entropic one, for testing membership of distributions of the form $\tilde{P}^{3:2}_p$. Here, we investigate how these two classes of tests compare when the size $D$ of the  local alphabets increases. We know from section \ref{SecTripartiteMonotone} that the semidefinite test is independent of $D$ for this particular family of distributions. It could thus potentially be the case that the entropic test would become stronger than the semidefinite test for sufficiently large alphabet sizes. This is indeed what we find in the numerical evaluation of the entropic test, which we  display in figure \ref{FigAlphabetSize}. 

As pointed out in section \ref{SecE1asymptotics}, all the entropic inequalities, apart from $E_1$, depend on the full tripartite distribution, while $E_1$ and the semidefinite test only utilize the bi- and mono-partite margins. We already know from the previous section that the test based on  $E_1$  always is weaker that the semidefinite test for the family $\tilde{P}^{3:D}_p$, but that it approaches the semidefinite test in the limit of large alphabet sizes $D$. As suggested by the plot in figure \ref{FigAlphabetSize}, the convergence is very slow. As an additional indication one may note that for an alphabet size of $D = 10^7$ the root of the equation $E_1(p) = 0$ is $p\approx 0.26$ while the limit is $p\approx 0.29$.

\section{\label{SecSummaryOutlook} Summary and outlook}

In this work we have considered the constraints imposed by a large class of causal structures on the covariance matrix of the observed variables. More specifically, we have shown that each bipartite DAG induces a decomposition that every covariance matrix resulting from the corresponding causal model has to satisfy. Such decompositions can be formulated in terms of semidefinite programs that allow for a straightforward and efficient computational treatment of the problem (as opposed to algebraic geometry solutions). A violation of the condition imposed by the bipartite DAG under test (or in other terms, the non-feasibility of the semidefinite program) thus implies that the observed covariance matrix is not compatible with it. We have also shown that every decomposition associated with a bipartite DAG can be realized by a causal model on that graph. 

Furthermore, we have made comparisons between the performance of the semidefinite test and tests based on information theoretic inequalities formulated in terms of entropies, where the results indicate that the semidefinite test outperforms the entropic test for moderate alphabet sizes of the random variables, while the latter become more powerful for large alphabet sizes.

These results open several directions for future research. Here, we have restricted attention to characterising the set of covariance matrices compatible with a given causal structure.
In real-world situations however, the covariance matrix is unknown and has to be estimated from a limited number of samples drawn from the underlying distribution.
This raises the question of how to turn the theory developed here into statistical hypothesis tests for a presumed causal stucture.
An obvious idea would be to construct a confidence region for the estimated covariance matrix and reject the hypothesis if the confidence region does not intersect the set compatible with the causal assumption.
We speculate, though, that it might be simpler to obtain statistically sound results by employing convex duality, as explained in the context of figure~\ref{fig:witness}.
Indeed, assume that $X$ is such that all compatible covariance matrices have non-negative inner product with $X$.
The inner product betweeen $X$ and the true covariance matrix is a scalar linear function of the distribution of the observable variables.
A one-sided statistical hypothesis test for $\mathrm{tr}\,\big(X\,\mathrm{Cov}(Y)\big) \leq 0$ with any desired significance level is therefore easy to construct.
It will automatically also test the causal hypothesis at the same significance level.
While any $X$ gives rise to such a test, their power to identify a given true incompatible distribution may very wildly.
One way of making an informed choice for $X$ would be as follows: Split the samples into two parts.
If the empirical covariance matrix of the first part is compatible with the hypothesis, accept. 
If not, the dual SDP (\ref{eqn:dual}) will identify a witness $X^\star$ that seperates the empirical matrix from the compatible set.
Now use the test based on $X^\star$ with the second part of the samples.
We leave the details to future work.

Another immediate question is to better understand the relation between the semidefinite and the entropic tests. Similarly, it would be highly desirable to combine our results with other tools that have very recently been proposed in order to characterize complex DAGs \cite{Chaves2016,Rosset2016,wolfe2016inflation}. On a more general level it is noteworthy that by restricting to covariance we turn a highly non-linear problem into what essentially is a convex optimization. Understanding how far this can be pushed (considering higher order moments, for instance) would certainly give us new geometric insights on the nature of this problem.
Since we here have focused on a setting where all correlations of observed variables are due to latent variables, it is very reasonable to ask if tests based on covariances can be extended to more general types of DAGs that do not have this bipartite structure.

From a more fundamental perspective our work may have implications for the current research program on the foundations of quantum physics. Bayesian networks have attracted growing attention as means to understand the role of causality in quantum mechanical systems \cite{Leifer2013,Fritz2012,Fritz2014,Henson2014,Chaves2015a,Piennar2014,ried2015quantum,Costa2016,horsman2016can}. One may thus ask whether the methods we have employed here can be generalized to the case of quantum causal structures, where for example some  nodes in the graph  represent quantum states without a classical analogue. Any positive results along this line would certainly be highly relevant in the context of quantum causal modeling and once more highlight the very fruitful interplay between the fields of causal inference and foundational aspects of quantum mechanics.

\begin{acknowledgments}

We thank Thomas Kahle and Johannes Textor for productive discussions during the early stages of this project.

This work has been supported by the Excellence Initiative of the German Federal and State Governments (Grants ZUK 43 and 81), the ARO under contract W911NF-14-1-0098 (Quantum Characterization, Verification, and Validation), and the DFG (SPP1798 CoSIP). 

\end{acknowledgments}
 
\bibliography{covariances}

\begin{thebibliography}{51}%
\makeatletter
\providecommand \@ifxundefined [1]{%
 \@ifx{#1\undefined}
}%
\providecommand \@ifnum [1]{%
 \ifnum #1\expandafter \@firstoftwo
 \else \expandafter \@secondoftwo
 \fi
}%
\providecommand \@ifx [1]{%
 \ifx #1\expandafter \@firstoftwo
 \else \expandafter \@secondoftwo
 \fi
}%
\providecommand \natexlab [1]{#1}%
\providecommand \enquote  [1]{``#1''}%
\providecommand \bibnamefont  [1]{#1}%
\providecommand \bibfnamefont [1]{#1}%
\providecommand \citenamefont [1]{#1}%
\providecommand \href@noop [0]{\@secondoftwo}%
\providecommand \href [0]{\begingroup \@sanitize@url \@href}%
\providecommand \@href[1]{\@@startlink{#1}\@@href}%
\providecommand \@@href[1]{\endgroup#1\@@endlink}%
\providecommand \@sanitize@url [0]{\catcode `\\12\catcode `\$12\catcode
  `\&12\catcode `\#12\catcode `\^12\catcode `\_12\catcode `\%12\relax}%
\providecommand \@@startlink[1]{}%
\providecommand \@@endlink[0]{}%
\providecommand \url  [0]{\begingroup\@sanitize@url \@url }%
\providecommand \@url [1]{\endgroup\@href {#1}{\urlprefix }}%
\providecommand \urlprefix  [0]{URL }%
\providecommand \Eprint [0]{\href }%
\providecommand \doibase [0]{http://dx.doi.org/}%
\providecommand \selectlanguage [0]{\@gobble}%
\providecommand \bibinfo  [0]{\@secondoftwo}%
\providecommand \bibfield  [0]{\@secondoftwo}%
\providecommand \translation [1]{[#1]}%
\providecommand \BibitemOpen [0]{}%
\providecommand \bibitemStop [0]{}%
\providecommand \bibitemNoStop [0]{.\EOS\space}%
\providecommand \EOS [0]{\spacefactor3000\relax}%
\providecommand \BibitemShut  [1]{\csname bibitem#1\endcsname}%
\let\auto@bib@innerbib\@empty
\bibitem [{\citenamefont {Pearl}(2009)}]{Pearlbook}%
  \BibitemOpen
  \bibfield  {author} {\bibinfo {author} {\bibfnamefont {J.}~\bibnamefont
  {Pearl}},\ }\href@noop {} {\emph {\bibinfo {title} {Causality}}}\ (\bibinfo
  {publisher} {Cambridge University Press, Cambridge},\ \bibinfo {year}
  {2009})\BibitemShut {NoStop}%
\bibitem [{\citenamefont {Spirtes}\ \emph {et~al.}(2001)\citenamefont
  {Spirtes}, \citenamefont {Glymour},\ and\ \citenamefont
  {Scheienes}}]{Spirtesbook}%
  \BibitemOpen
  \bibfield  {author} {\bibinfo {author} {\bibfnamefont {P.}~\bibnamefont
  {Spirtes}}, \bibinfo {author} {\bibfnamefont {N.}~\bibnamefont {Glymour}}, \
  and\ \bibinfo {author} {\bibfnamefont {R.}~\bibnamefont {Scheienes}},\
  }\href@noop {} {\emph {\bibinfo {title} {Causation, Prediction, and Search,
  2nd ed.}}}\ (\bibinfo  {publisher} {The MIT Press},\ \bibinfo {year}
  {2001})\BibitemShut {NoStop}%
\bibitem [{\citenamefont {Friedman}(2004)}]{friedman2004inferring}%
  \BibitemOpen
  \bibfield  {author} {\bibinfo {author} {\bibfnamefont {Nir}\ \bibnamefont
  {Friedman}},\ }\bibfield  {title} {\enquote {\bibinfo {title} {Inferring
  cellular networks using probabilistic graphical models},}\ }\href@noop {}
  {\bibfield  {journal} {\bibinfo  {journal} {Science}\ }\textbf {\bibinfo
  {volume} {303}},\ \bibinfo {pages} {799--805} (\bibinfo {year}
  {2004})}\BibitemShut {NoStop}%
\bibitem [{\citenamefont {Ver~Steeg}\ and\ \citenamefont
  {Galstyan}(2011)}]{Steeg2011}%
  \BibitemOpen
  \bibfield  {author} {\bibinfo {author} {\bibfnamefont {G.}~\bibnamefont
  {Ver~Steeg}}\ and\ \bibinfo {author} {\bibfnamefont {A.}~\bibnamefont
  {Galstyan}},\ }\bibfield  {title} {\enquote {\bibinfo {title} {A sequence of
  relaxations constraining hidden variable models},}\ }in\ \href@noop {} {\emph
  {\bibinfo {booktitle} {Proceedings of the 27th conference on Uncertainty in
  Artificial Intelligence}}}\ (\bibinfo {year} {2011})\BibitemShut {NoStop}%
\bibitem [{\citenamefont {Leifer}\ and\ \citenamefont
  {Spekkens}(2013)}]{Leifer2013}%
  \BibitemOpen
  \bibfield  {author} {\bibinfo {author} {\bibfnamefont {M.~S.}\ \bibnamefont
  {Leifer}}\ and\ \bibinfo {author} {\bibfnamefont {Robert~W.}\ \bibnamefont
  {Spekkens}},\ }\bibfield  {title} {\enquote {\bibinfo {title} {Towards a
  formulation of quantum theory as a causally neutral theory of bayesian
  inference},}\ }\href {\doibase 10.1103/PhysRevA.88.052130} {\bibfield
  {journal} {\bibinfo  {journal} {Phys. Rev. A}\ }\textbf {\bibinfo {volume}
  {88}},\ \bibinfo {pages} {052130} (\bibinfo {year} {2013})}\BibitemShut
  {NoStop}%
\bibitem [{\citenamefont {Fritz}(2012)}]{Fritz2012}%
  \BibitemOpen
  \bibfield  {author} {\bibinfo {author} {\bibfnamefont {T.}~\bibnamefont
  {Fritz}},\ }\bibfield  {title} {\enquote {\bibinfo {title} {Beyond bell's
  theorem: correlation scenarios},}\ }\href
  {http://stacks.iop.org/1367-2630/14/i=10/a=103001} {\bibfield  {journal}
  {\bibinfo  {journal} {New J. Phys.}\ }\textbf {\bibinfo {volume} {14}},\
  \bibinfo {pages} {103001} (\bibinfo {year} {2012})}\BibitemShut {NoStop}%
\bibitem [{\citenamefont {Fritz}(2016)}]{Fritz2014}%
  \BibitemOpen
  \bibfield  {author} {\bibinfo {author} {\bibfnamefont {Tobias}\ \bibnamefont
  {Fritz}},\ }\bibfield  {title} {\enquote {\bibinfo {title} {Beyond bell’s
  theorem ii: Scenarios with arbitrary causal structure},}\ }\href {\doibase
  10.1007/s00220-015-2495-5} {\bibfield  {journal} {\bibinfo  {journal}
  {Communications in Mathematical Physics}\ }\textbf {\bibinfo {volume}
  {341}},\ \bibinfo {pages} {391--434} (\bibinfo {year} {2016})}\BibitemShut
  {NoStop}%
\bibitem [{\citenamefont {Henson}\ \emph {et~al.}(2014)\citenamefont {Henson},
  \citenamefont {Lal},\ and\ \citenamefont {Pusey}}]{Henson2014}%
  \BibitemOpen
  \bibfield  {author} {\bibinfo {author} {\bibfnamefont {Joe}\ \bibnamefont
  {Henson}}, \bibinfo {author} {\bibfnamefont {Raymond}\ \bibnamefont {Lal}}, \
  and\ \bibinfo {author} {\bibfnamefont {Matthew~F}\ \bibnamefont {Pusey}},\
  }\bibfield  {title} {\enquote {\bibinfo {title} {Theory-independent limits on
  correlations from generalized bayesian networks},}\ }\href
  {http://stacks.iop.org/1367-2630/16/i=11/a=113043} {\bibfield  {journal}
  {\bibinfo  {journal} {New J. Phys.}\ }\textbf {\bibinfo {volume} {16}},\
  \bibinfo {pages} {113043} (\bibinfo {year} {2014})}\BibitemShut {NoStop}%
\bibitem [{\citenamefont {Chaves}\ \emph
  {et~al.}(2015{\natexlab{a}})\citenamefont {Chaves}, \citenamefont {Majenz},\
  and\ \citenamefont {Gross}}]{Chaves2015a}%
  \BibitemOpen
  \bibfield  {author} {\bibinfo {author} {\bibfnamefont {Rafael}\ \bibnamefont
  {Chaves}}, \bibinfo {author} {\bibfnamefont {Christian}\ \bibnamefont
  {Majenz}}, \ and\ \bibinfo {author} {\bibfnamefont {David}\ \bibnamefont
  {Gross}},\ }\bibfield  {title} {\enquote {\bibinfo {title}
  {Information--theoretic implications of quantum causal structures},}\ }\href
  {\doibase 10.1038/ncomms6766} {\bibfield  {journal} {\bibinfo  {journal}
  {Nat. Commun.}\ }\textbf {\bibinfo {volume} {6}},\ \bibinfo {pages} {5766}
  (\bibinfo {year} {2015}{\natexlab{a}})}\BibitemShut {NoStop}%
\bibitem [{\citenamefont {Pienaar}\ and\ \citenamefont
  {Brukner}(2015)}]{Piennar2014}%
  \BibitemOpen
  \bibfield  {author} {\bibinfo {author} {\bibfnamefont {Jacques}\ \bibnamefont
  {Pienaar}}\ and\ \bibinfo {author} {\bibfnamefont {Caslav}\ \bibnamefont
  {Brukner}},\ }\bibfield  {title} {\enquote {\bibinfo {title} {A
  graph-separation theorem for quantum causal models},}\ }\href
  {http://stacks.iop.org/1367-2630/17/i=7/a=073020} {\bibfield  {journal}
  {\bibinfo  {journal} {New J. Phys.}\ }\textbf {\bibinfo {volume} {17}},\
  \bibinfo {pages} {073020} (\bibinfo {year} {2015})}\BibitemShut {NoStop}%
\bibitem [{\citenamefont {Ried}\ \emph {et~al.}(2015)\citenamefont {Ried},
  \citenamefont {Agnew}, \citenamefont {Vermeyden}, \citenamefont {Janzing},
  \citenamefont {Spekkens},\ and\ \citenamefont {Resch}}]{ried2015quantum}%
  \BibitemOpen
  \bibfield  {author} {\bibinfo {author} {\bibfnamefont {Katja}\ \bibnamefont
  {Ried}}, \bibinfo {author} {\bibfnamefont {Megan}\ \bibnamefont {Agnew}},
  \bibinfo {author} {\bibfnamefont {Lydia}\ \bibnamefont {Vermeyden}}, \bibinfo
  {author} {\bibfnamefont {Dominik}\ \bibnamefont {Janzing}}, \bibinfo {author}
  {\bibfnamefont {Robert~W}\ \bibnamefont {Spekkens}}, \ and\ \bibinfo {author}
  {\bibfnamefont {Kevin~J}\ \bibnamefont {Resch}},\ }\bibfield  {title}
  {\enquote {\bibinfo {title} {A quantum advantage for inferring causal
  structure},}\ }\href@noop {} {\bibfield  {journal} {\bibinfo  {journal}
  {Nature Physics}\ }\textbf {\bibinfo {volume} {11}},\ \bibinfo {pages}
  {414--420} (\bibinfo {year} {2015})}\BibitemShut {NoStop}%
\bibitem [{\citenamefont {Costa}\ and\ \citenamefont
  {Shrapnel}(2016)}]{Costa2016}%
  \BibitemOpen
  \bibfield  {author} {\bibinfo {author} {\bibfnamefont {Fabio}\ \bibnamefont
  {Costa}}\ and\ \bibinfo {author} {\bibfnamefont {Sally}\ \bibnamefont
  {Shrapnel}},\ }\bibfield  {title} {\enquote {\bibinfo {title} {Quantum causal
  modelling},}\ }\href {http://stacks.iop.org/1367-2630/18/i=6/a=063032}
  {\bibfield  {journal} {\bibinfo  {journal} {New Journal of Physics}\ }\textbf
  {\bibinfo {volume} {18}},\ \bibinfo {pages} {063032} (\bibinfo {year}
  {2016})}\BibitemShut {NoStop}%
\bibitem [{\citenamefont {Horsman}\ \emph {et~al.}(2016)\citenamefont
  {Horsman}, \citenamefont {Heunen}, \citenamefont {Pusey}, \citenamefont
  {Barrett},\ and\ \citenamefont {Spekkens}}]{horsman2016can}%
  \BibitemOpen
  \bibfield  {author} {\bibinfo {author} {\bibfnamefont {Dominic}\ \bibnamefont
  {Horsman}}, \bibinfo {author} {\bibfnamefont {Chris}\ \bibnamefont {Heunen}},
  \bibinfo {author} {\bibfnamefont {Matthew~F}\ \bibnamefont {Pusey}}, \bibinfo
  {author} {\bibfnamefont {Jonathan}\ \bibnamefont {Barrett}}, \ and\ \bibinfo
  {author} {\bibfnamefont {Robert~W}\ \bibnamefont {Spekkens}},\ }\bibfield
  {title} {\enquote {\bibinfo {title} {Can a quantum state over time resemble a
  quantum state at a single time?}}\ }\href@noop {} {\bibfield  {journal}
  {\bibinfo  {journal} {arXiv preprint arXiv:1607.03637}\ } (\bibinfo {year}
  {2016})}\BibitemShut {NoStop}%
\bibitem [{\citenamefont {Pitowsky}(1991)}]{Pitowsky1991}%
  \BibitemOpen
  \bibfield  {author} {\bibinfo {author} {\bibfnamefont {Itamar}\ \bibnamefont
  {Pitowsky}},\ }\bibfield  {title} {\enquote {\bibinfo {title} {Correlation
  polytopes: Their geometry and complexity},}\ }\href {\doibase
  10.1007/BF01594946} {\bibfield  {journal} {\bibinfo  {journal} {Mathematical
  Programming}\ }\textbf {\bibinfo {volume} {50}},\ \bibinfo {pages} {395--414}
  (\bibinfo {year} {1991})}\BibitemShut {NoStop}%
\bibitem [{\citenamefont {Pearl}(1995)}]{Pearl1995}%
  \BibitemOpen
  \bibfield  {author} {\bibinfo {author} {\bibfnamefont {J.}~\bibnamefont
  {Pearl}},\ }\bibfield  {title} {\enquote {\bibinfo {title} {On the
  testability of causal models with latent and instrumental variables},}\ }in\
  \href@noop {} {\emph {\bibinfo {booktitle} {Proceedings of the 11th
  conference on Uncertainty in Artificial Intelligence}}}\ (\bibinfo {year}
  {1995})\ pp.\ \bibinfo {pages} {435--443}\BibitemShut {NoStop}%
\bibitem [{\citenamefont {Geiger}\ and\ \citenamefont
  {Meek}(1999)}]{Geiger1999}%
  \BibitemOpen
  \bibfield  {author} {\bibinfo {author} {\bibfnamefont {D.}~\bibnamefont
  {Geiger}}\ and\ \bibinfo {author} {\bibfnamefont {C.}~\bibnamefont {Meek}},\
  }\bibfield  {title} {\enquote {\bibinfo {title} {Quantifier elimination for
  statistical problems},}\ }in\ \href@noop {} {\emph {\bibinfo {booktitle}
  {Proceedings of the 15th conference on Uncertainty in Artificial
  Intelligence}}}\ (\bibinfo {year} {1999})\ pp.\ \bibinfo {pages}
  {226--235}\BibitemShut {NoStop}%
\bibitem [{\citenamefont {Bonet}(2001)}]{Bonet2001}%
  \BibitemOpen
  \bibfield  {author} {\bibinfo {author} {\bibfnamefont {B.}~\bibnamefont
  {Bonet}},\ }\bibfield  {title} {\enquote {\bibinfo {title} {Instrumentality
  tests revisited},}\ }in\ \href@noop {} {\emph {\bibinfo {booktitle}
  {Proceedings of the 17th Conference on Uncertainty in Artificial
  Intelligence}}}\ (\bibinfo {year} {2001})\ pp.\ \bibinfo {pages}
  {48--55}\BibitemShut {NoStop}%
\bibitem [{\citenamefont {Garcia}\ \emph {et~al.}(2005)\citenamefont {Garcia},
  \citenamefont {Stillman},\ and\ \citenamefont {Sturmfels}}]{Garcia2005}%
  \BibitemOpen
  \bibfield  {author} {\bibinfo {author} {\bibfnamefont {L.~D.}\ \bibnamefont
  {Garcia}}, \bibinfo {author} {\bibfnamefont {M.}~\bibnamefont {Stillman}}, \
  and\ \bibinfo {author} {\bibfnamefont {B.}~\bibnamefont {Sturmfels}},\
  }\bibfield  {title} {\enquote {\bibinfo {title} {Algebraic geometry of
  bayesian networks},}\ }\href@noop {} {\bibfield  {journal} {\bibinfo
  {journal} {Journal of Symbolic Computation}\ }\textbf {\bibinfo {volume}
  {39}},\ \bibinfo {pages} {331--355} (\bibinfo {year} {2005})}\BibitemShut
  {NoStop}%
\bibitem [{\citenamefont {Kang}\ and\ \citenamefont {Tian}(2006)}]{Kang2006}%
  \BibitemOpen
  \bibfield  {author} {\bibinfo {author} {\bibfnamefont {C.}~\bibnamefont
  {Kang}}\ and\ \bibinfo {author} {\bibfnamefont {J.}~\bibnamefont {Tian}},\
  }\bibfield  {title} {\enquote {\bibinfo {title} {Inequality constraints in
  causal models with hidden variables},}\ }in\ \href@noop {} {\emph {\bibinfo
  {booktitle} {Proceedings of the 22nd Conference on Uncertainty in Artificial
  Intelligence}}}\ (\bibinfo {year} {2006})\ pp.\ \bibinfo {pages}
  {233--240}\BibitemShut {NoStop}%
\bibitem [{\citenamefont {Kang}\ and\ \citenamefont {Tian}(2007)}]{Kang2007}%
  \BibitemOpen
  \bibfield  {author} {\bibinfo {author} {\bibfnamefont {C.}~\bibnamefont
  {Kang}}\ and\ \bibinfo {author} {\bibfnamefont {J.}~\bibnamefont {Tian}},\
  }\bibfield  {title} {\enquote {\bibinfo {title} {Polynomial constraints in
  causal bayesian networks},}\ }in\ \href@noop {} {\emph {\bibinfo {booktitle}
  {Proceedings of the 23rd Conference on Uncertainty in Artificial
  Intelligence}}}\ (\bibinfo {year} {2007})\ pp.\ \bibinfo {pages}
  {200--208}\BibitemShut {NoStop}%
\bibitem [{\citenamefont {Evans}(2012)}]{evans2012graphical}%
  \BibitemOpen
  \bibfield  {author} {\bibinfo {author} {\bibfnamefont {Robin~J}\ \bibnamefont
  {Evans}},\ }\bibfield  {title} {\enquote {\bibinfo {title} {Graphical methods
  for inequality constraints in marginalized dags},}\ }in\ \href@noop {} {\emph
  {\bibinfo {booktitle} {2012 IEEE International Workshop on Machine Learning
  for Signal Processing}}}\ (\bibinfo {organization} {IEEE},\ \bibinfo {year}
  {2012})\ pp.\ \bibinfo {pages} {1--6}\BibitemShut {NoStop}%
\bibitem [{\citenamefont {Lee}\ and\ \citenamefont
  {Spekkens}(2015)}]{lee2015causal}%
  \BibitemOpen
  \bibfield  {author} {\bibinfo {author} {\bibfnamefont {Ciar{\'a}n~M}\
  \bibnamefont {Lee}}\ and\ \bibinfo {author} {\bibfnamefont {Robert~W}\
  \bibnamefont {Spekkens}},\ }\bibfield  {title} {\enquote {\bibinfo {title}
  {Causal inference via algebraic geometry: necessary and sufficient conditions
  for the feasibility of discrete causal models},}\ }\href@noop {} {\bibfield
  {journal} {\bibinfo  {journal} {arXiv preprint arXiv:1506.03880}\ } (\bibinfo
  {year} {2015})}\BibitemShut {NoStop}%
\bibitem [{\citenamefont {Chaves}(2016)}]{Chaves2016}%
  \BibitemOpen
  \bibfield  {author} {\bibinfo {author} {\bibfnamefont {Rafael}\ \bibnamefont
  {Chaves}},\ }\bibfield  {title} {\enquote {\bibinfo {title} {Polynomial bell
  inequalities},}\ }\href {\doibase 10.1103/PhysRevLett.116.010402} {\bibfield
  {journal} {\bibinfo  {journal} {Phys. Rev. Lett.}\ }\textbf {\bibinfo
  {volume} {116}},\ \bibinfo {pages} {010402} (\bibinfo {year}
  {2016})}\BibitemShut {NoStop}%
\bibitem [{\citenamefont {Rosset}\ \emph {et~al.}(2016)\citenamefont {Rosset},
  \citenamefont {Branciard}, \citenamefont {Barnea}, \citenamefont {P\"utz},
  \citenamefont {Brunner},\ and\ \citenamefont {Gisin}}]{Rosset2016}%
  \BibitemOpen
  \bibfield  {author} {\bibinfo {author} {\bibfnamefont {Denis}\ \bibnamefont
  {Rosset}}, \bibinfo {author} {\bibfnamefont {Cyril}\ \bibnamefont
  {Branciard}}, \bibinfo {author} {\bibfnamefont {Tomer~Jack}\ \bibnamefont
  {Barnea}}, \bibinfo {author} {\bibfnamefont {Gilles}\ \bibnamefont {P\"utz}},
  \bibinfo {author} {\bibfnamefont {Nicolas}\ \bibnamefont {Brunner}}, \ and\
  \bibinfo {author} {\bibfnamefont {Nicolas}\ \bibnamefont {Gisin}},\
  }\bibfield  {title} {\enquote {\bibinfo {title} {Nonlinear bell inequalities
  tailored for quantum networks},}\ }\href {\doibase
  10.1103/PhysRevLett.116.010403} {\bibfield  {journal} {\bibinfo  {journal}
  {Phys. Rev. Lett.}\ }\textbf {\bibinfo {volume} {116}},\ \bibinfo {pages}
  {010403} (\bibinfo {year} {2016})}\BibitemShut {NoStop}%
\bibitem [{\citenamefont {Wolfe}\ \emph {et~al.}(2016)\citenamefont {Wolfe},
  \citenamefont {Spekkens},\ and\ \citenamefont {Fritz}}]{wolfe2016inflation}%
  \BibitemOpen
  \bibfield  {author} {\bibinfo {author} {\bibfnamefont {Elie}\ \bibnamefont
  {Wolfe}}, \bibinfo {author} {\bibfnamefont {Robert~W}\ \bibnamefont
  {Spekkens}}, \ and\ \bibinfo {author} {\bibfnamefont {Tobias}\ \bibnamefont
  {Fritz}},\ }\bibfield  {title} {\enquote {\bibinfo {title} {The inflation
  technique for causal inference with latent variables},}\ }\href@noop {}
  {\bibfield  {journal} {\bibinfo  {journal} {arXiv preprint arXiv:1609.00672}\
  } (\bibinfo {year} {2016})}\BibitemShut {NoStop}%
\bibitem [{\citenamefont {Moritz}\ \emph {et~al.}(2014)\citenamefont {Moritz},
  \citenamefont {Reichardt},\ and\ \citenamefont
  {Ay}}]{moritz2014discriminating}%
  \BibitemOpen
  \bibfield  {author} {\bibinfo {author} {\bibfnamefont {Philipp}\ \bibnamefont
  {Moritz}}, \bibinfo {author} {\bibfnamefont {J{\"o}rg}\ \bibnamefont
  {Reichardt}}, \ and\ \bibinfo {author} {\bibfnamefont {Nihat}\ \bibnamefont
  {Ay}},\ }\bibfield  {title} {\enquote {\bibinfo {title} {Discriminating
  between causal structures in bayesian networks given partial observations},}\
  }\href@noop {} {\bibfield  {journal} {\bibinfo  {journal} {Kybernetika}\
  }\textbf {\bibinfo {volume} {50}} (\bibinfo {year} {2014})}\BibitemShut
  {NoStop}%
\bibitem [{\citenamefont {Chaves}\ \emph
  {et~al.}(2014{\natexlab{a}})\citenamefont {Chaves}, \citenamefont {Luft},\
  and\ \citenamefont {Gross}}]{Chaves2014}%
  \BibitemOpen
  \bibfield  {author} {\bibinfo {author} {\bibfnamefont {Rafael}\ \bibnamefont
  {Chaves}}, \bibinfo {author} {\bibfnamefont {Lukas}\ \bibnamefont {Luft}}, \
  and\ \bibinfo {author} {\bibfnamefont {David}\ \bibnamefont {Gross}},\
  }\bibfield  {title} {\enquote {\bibinfo {title} {Causal structures from
  entropic information: geometry and novel scenarios},}\ }\href
  {http://stacks.iop.org/1367-2630/16/i=4/a=043001} {\bibfield  {journal}
  {\bibinfo  {journal} {New Journal of Physics}\ }\textbf {\bibinfo {volume}
  {16}},\ \bibinfo {pages} {043001} (\bibinfo {year}
  {2014}{\natexlab{a}})}\BibitemShut {NoStop}%
\bibitem [{\citenamefont {Chaves}\ \emph
  {et~al.}(2014{\natexlab{b}})\citenamefont {Chaves}, \citenamefont {Luft},
  \citenamefont {Maciel}, \citenamefont {Gross}, \citenamefont {Janzing},\ and\
  \citenamefont {Sch\"olkopf}}]{Chaves2014b}%
  \BibitemOpen
  \bibfield  {author} {\bibinfo {author} {\bibfnamefont {R.}~\bibnamefont
  {Chaves}}, \bibinfo {author} {\bibfnamefont {L.}~\bibnamefont {Luft}},
  \bibinfo {author} {\bibfnamefont {T.~O.}\ \bibnamefont {Maciel}}, \bibinfo
  {author} {\bibfnamefont {D.}~\bibnamefont {Gross}}, \bibinfo {author}
  {\bibfnamefont {D.}~\bibnamefont {Janzing}}, \ and\ \bibinfo {author}
  {\bibfnamefont {B.}~\bibnamefont {Sch\"olkopf}},\ }\bibfield  {title}
  {\enquote {\bibinfo {title} {Inferring latent structures via information
  inequalities},}\ }\href {http://arxiv.org/abs/1407.2256} {\bibfield
  {journal} {\bibinfo  {journal} {Proceedings of the 30th Conference on
  Uncertainty in Artificial Intelligence}\ ,\ \bibinfo {pages} {112--121}}
  (\bibinfo {year} {2014}{\natexlab{b}})}\BibitemShut {NoStop}%
\bibitem [{\citenamefont {Steudel}\ and\ \citenamefont
  {Ay}(2015)}]{steudel2015information}%
  \BibitemOpen
  \bibfield  {author} {\bibinfo {author} {\bibfnamefont {Bastian}\ \bibnamefont
  {Steudel}}\ and\ \bibinfo {author} {\bibfnamefont {Nihat}\ \bibnamefont
  {Ay}},\ }\bibfield  {title} {\enquote {\bibinfo {title}
  {Information-theoretic inference of common ancestors},}\ }\href@noop {}
  {\bibfield  {journal} {\bibinfo  {journal} {Entropy}\ }\textbf {\bibinfo
  {volume} {17}},\ \bibinfo {pages} {2304--2327} (\bibinfo {year}
  {2015})}\BibitemShut {NoStop}%
\bibitem [{\citenamefont {Weilenmann}\ and\ \citenamefont
  {Colbeck}(2016)}]{weilenmann2016non}%
  \BibitemOpen
  \bibfield  {author} {\bibinfo {author} {\bibfnamefont {Mirjam}\ \bibnamefont
  {Weilenmann}}\ and\ \bibinfo {author} {\bibfnamefont {Roger}\ \bibnamefont
  {Colbeck}},\ }\bibfield  {title} {\enquote {\bibinfo {title} {Non-shannon
  inequalities in the entropy vector approach to causal structures},}\
  }\href@noop {} {\bibfield  {journal} {\bibinfo  {journal} {arXiv preprint
  arXiv:1605.02078}\ } (\bibinfo {year} {2016})}\BibitemShut {NoStop}%
\bibitem [{\citenamefont {Braunstein}\ and\ \citenamefont
  {Caves}(1988)}]{Braunstein1988}%
  \BibitemOpen
  \bibfield  {author} {\bibinfo {author} {\bibfnamefont {Samuel~L.}\
  \bibnamefont {Braunstein}}\ and\ \bibinfo {author} {\bibfnamefont
  {Carlton~M.}\ \bibnamefont {Caves}},\ }\bibfield  {title} {\enquote {\bibinfo
  {title} {Information-theoretic bell inequalities},}\ }\href {\doibase
  10.1103/PhysRevLett.61.662} {\bibfield  {journal} {\bibinfo  {journal} {Phys.
  Rev. Lett.}\ }\textbf {\bibinfo {volume} {61}},\ \bibinfo {pages} {662--665}
  (\bibinfo {year} {1988})}\BibitemShut {NoStop}%
\bibitem [{\citenamefont {Cerf}\ and\ \citenamefont {Adami}(1997)}]{Cerf1997}%
  \BibitemOpen
  \bibfield  {author} {\bibinfo {author} {\bibfnamefont {N.~J.}\ \bibnamefont
  {Cerf}}\ and\ \bibinfo {author} {\bibfnamefont {C.}~\bibnamefont {Adami}},\
  }\bibfield  {title} {\enquote {\bibinfo {title} {Entropic bell
  inequalities},}\ }\href {\doibase 10.1103/PhysRevA.55.3371} {\bibfield
  {journal} {\bibinfo  {journal} {Phys. Rev. A}\ }\textbf {\bibinfo {volume}
  {55}},\ \bibinfo {pages} {3371--3374} (\bibinfo {year} {1997})}\BibitemShut
  {NoStop}%
\bibitem [{\citenamefont {Chaves}\ and\ \citenamefont
  {Fritz}(2012)}]{Chaves2012}%
  \BibitemOpen
  \bibfield  {author} {\bibinfo {author} {\bibfnamefont {Rafael}\ \bibnamefont
  {Chaves}}\ and\ \bibinfo {author} {\bibfnamefont {Tobias}\ \bibnamefont
  {Fritz}},\ }\bibfield  {title} {\enquote {\bibinfo {title} {Entropic approach
  to local realism and noncontextuality},}\ }\href {\doibase
  10.1103/PhysRevA.85.032113} {\bibfield  {journal} {\bibinfo  {journal} {Phys.
  Rev. A}\ }\textbf {\bibinfo {volume} {85}},\ \bibinfo {pages} {032113}
  (\bibinfo {year} {2012})}\BibitemShut {NoStop}%
\bibitem [{\citenamefont {Fritz}\ and\ \citenamefont
  {Chaves}(2013)}]{FritzChaves2013}%
  \BibitemOpen
  \bibfield  {author} {\bibinfo {author} {\bibfnamefont {T.}~\bibnamefont
  {Fritz}}\ and\ \bibinfo {author} {\bibfnamefont {R.}~\bibnamefont {Chaves}},\
  }\bibfield  {title} {\enquote {\bibinfo {title} {Entropic inequalities and
  marginal problems},}\ }\href {\doibase 10.1109/TIT.2012.2222863} {\bibfield
  {journal} {\bibinfo  {journal} {IEEE Trans. Inform. Theory}\ }\textbf
  {\bibinfo {volume} {59}},\ \bibinfo {pages} {803} (\bibinfo {year}
  {2013})}\BibitemShut {NoStop}%
\bibitem [{\citenamefont {Chaves}(2013)}]{Chaves2013entropic}%
  \BibitemOpen
  \bibfield  {author} {\bibinfo {author} {\bibfnamefont {Rafael}\ \bibnamefont
  {Chaves}},\ }\bibfield  {title} {\enquote {\bibinfo {title} {Entropic
  inequalities as a necessary and sufficient condition to noncontextuality and
  locality},}\ }\href {\doibase 10.1103/PhysRevA.87.022102} {\bibfield
  {journal} {\bibinfo  {journal} {Phys. Rev. A}\ }\textbf {\bibinfo {volume}
  {87}},\ \bibinfo {pages} {022102} (\bibinfo {year} {2013})}\BibitemShut
  {NoStop}%
\bibitem [{\citenamefont {Chaves}\ \emph
  {et~al.}(2015{\natexlab{b}})\citenamefont {Chaves}, \citenamefont {Brask},\
  and\ \citenamefont {Brunner}}]{Chaves2015entropy}%
  \BibitemOpen
  \bibfield  {author} {\bibinfo {author} {\bibfnamefont {Rafael}\ \bibnamefont
  {Chaves}}, \bibinfo {author} {\bibfnamefont {Jonatan~Bohr}\ \bibnamefont
  {Brask}}, \ and\ \bibinfo {author} {\bibfnamefont {Nicolas}\ \bibnamefont
  {Brunner}},\ }\bibfield  {title} {\enquote {\bibinfo {title}
  {Device-independent tests of entropy},}\ }\href {\doibase
  10.1103/PhysRevLett.115.110501} {\bibfield  {journal} {\bibinfo  {journal}
  {Phys. Rev. Lett.}\ }\textbf {\bibinfo {volume} {115}},\ \bibinfo {pages}
  {110501} (\bibinfo {year} {2015}{\natexlab{b}})}\BibitemShut {NoStop}%
\bibitem [{\citenamefont {Chaves}\ and\ \citenamefont
  {Budroni}(2016)}]{Chaves2016entropic}%
  \BibitemOpen
  \bibfield  {author} {\bibinfo {author} {\bibfnamefont {Rafael}\ \bibnamefont
  {Chaves}}\ and\ \bibinfo {author} {\bibfnamefont {Costantino}\ \bibnamefont
  {Budroni}},\ }\bibfield  {title} {\enquote {\bibinfo {title} {Entropic
  nonsignaling correlations},}\ }\href {\doibase
  10.1103/PhysRevLett.116.240501} {\bibfield  {journal} {\bibinfo  {journal}
  {Phys. Rev. Lett.}\ }\textbf {\bibinfo {volume} {116}},\ \bibinfo {pages}
  {240501} (\bibinfo {year} {2016})}\BibitemShut {NoStop}%
\bibitem [{\citenamefont {Evans}(2016)}]{evans2015graphs}%
  \BibitemOpen
  \bibfield  {author} {\bibinfo {author} {\bibfnamefont {Robin~J}\ \bibnamefont
  {Evans}},\ }\bibfield  {title} {\enquote {\bibinfo {title} {Graphs for
  margins of bayesian networks},}\ }\href@noop {} {\bibfield  {journal}
  {\bibinfo  {journal} {Scandinavian Journal of Statistics}\ }\textbf {\bibinfo
  {volume} {3}},\ \bibinfo {pages} {625--648} (\bibinfo {year}
  {2016})}\BibitemShut {NoStop}%
\bibitem [{\citenamefont {Branciard}\ \emph {et~al.}(2010)\citenamefont
  {Branciard}, \citenamefont {Gisin},\ and\ \citenamefont
  {Pironio}}]{Branciard2010}%
  \BibitemOpen
  \bibfield  {author} {\bibinfo {author} {\bibfnamefont {C.}~\bibnamefont
  {Branciard}}, \bibinfo {author} {\bibfnamefont {N.}~\bibnamefont {Gisin}}, \
  and\ \bibinfo {author} {\bibfnamefont {S.}~\bibnamefont {Pironio}},\
  }\bibfield  {title} {\enquote {\bibinfo {title} {Characterizing the nonlocal
  correlations created via entanglement swapping},}\ }\href {\doibase
  10.1103/PhysRevLett.104.170401} {\bibfield  {journal} {\bibinfo  {journal}
  {Phys. Rev. Lett.}\ }\textbf {\bibinfo {volume} {104}},\ \bibinfo {pages}
  {170401} (\bibinfo {year} {2010})}\BibitemShut {NoStop}%
\bibitem [{\citenamefont {Branciard}\ \emph {et~al.}(2012)\citenamefont
  {Branciard}, \citenamefont {Rosset}, \citenamefont {Gisin},\ and\
  \citenamefont {Pironio}}]{Branciard2012}%
  \BibitemOpen
  \bibfield  {author} {\bibinfo {author} {\bibfnamefont {Cyril}\ \bibnamefont
  {Branciard}}, \bibinfo {author} {\bibfnamefont {Denis}\ \bibnamefont
  {Rosset}}, \bibinfo {author} {\bibfnamefont {Nicolas}\ \bibnamefont {Gisin}},
  \ and\ \bibinfo {author} {\bibfnamefont {Stefano}\ \bibnamefont {Pironio}},\
  }\bibfield  {title} {\enquote {\bibinfo {title} {Bilocal versus nonbilocal
  correlations in entanglement-swapping experiments},}\ }\href {\doibase
  10.1103/PhysRevA.85.032119} {\bibfield  {journal} {\bibinfo  {journal} {Phys.
  Rev. A}\ }\textbf {\bibinfo {volume} {85}},\ \bibinfo {pages} {032119}
  (\bibinfo {year} {2012})}\BibitemShut {NoStop}%
\bibitem [{\citenamefont {Tavakoli}\ \emph {et~al.}(2014)\citenamefont
  {Tavakoli}, \citenamefont {Skrzypczyk}, \citenamefont {Cavalcanti},\ and\
  \citenamefont {Ac{\'\i}n}}]{tavakoli2014nonlocal}%
  \BibitemOpen
  \bibfield  {author} {\bibinfo {author} {\bibfnamefont {Armin}\ \bibnamefont
  {Tavakoli}}, \bibinfo {author} {\bibfnamefont {Paul}\ \bibnamefont
  {Skrzypczyk}}, \bibinfo {author} {\bibfnamefont {Daniel}\ \bibnamefont
  {Cavalcanti}}, \ and\ \bibinfo {author} {\bibfnamefont {Antonio}\
  \bibnamefont {Ac{\'\i}n}},\ }\bibfield  {title} {\enquote {\bibinfo {title}
  {Nonlocal correlations in the star-network configuration},}\ }\href@noop {}
  {\bibfield  {journal} {\bibinfo  {journal} {Physical Review A}\ }\textbf
  {\bibinfo {volume} {90}},\ \bibinfo {pages} {062109} (\bibinfo {year}
  {2014})}\BibitemShut {NoStop}%
\bibitem [{\citenamefont {Wood}\ and\ \citenamefont
  {Spekkens}(2015)}]{Spekkens2015}%
  \BibitemOpen
  \bibfield  {author} {\bibinfo {author} {\bibfnamefont {Christopher~J}\
  \bibnamefont {Wood}}\ and\ \bibinfo {author} {\bibfnamefont {Robert~W}\
  \bibnamefont {Spekkens}},\ }\bibfield  {title} {\enquote {\bibinfo {title}
  {The lesson of causal discovery algorithms for quantum correlations: causal
  explanations of bell-inequality violations require fine-tuning},}\ }\href
  {http://stacks.iop.org/1367-2630/17/i=3/a=033002} {\bibfield  {journal}
  {\bibinfo  {journal} {New J. Phys.}\ }\textbf {\bibinfo {volume} {17}},\
  \bibinfo {pages} {033002} (\bibinfo {year} {2015})}\BibitemShut {NoStop}%
\bibitem [{\citenamefont {Bell}(1964)}]{Bell1964}%
  \BibitemOpen
  \bibfield  {author} {\bibinfo {author} {\bibfnamefont {J.~S.}\ \bibnamefont
  {Bell}},\ }\bibfield  {title} {\enquote {\bibinfo {title} {On the
  {E}instein--{P}odolsky--{R}osen paradox},}\ }\href@noop {} {\bibfield
  {journal} {\bibinfo  {journal} {Physics}\ }\textbf {\bibinfo {volume} {1}},\
  \bibinfo {pages} {195} (\bibinfo {year} {1964})}\BibitemShut {NoStop}%
\bibitem [{\citenamefont {Saunders}\ \emph {et~al.}(2016)\citenamefont
  {Saunders}, \citenamefont {Bennet}, \citenamefont {Branciard},\ and\
  \citenamefont {Pryde}}]{saunders2016experimental}%
  \BibitemOpen
  \bibfield  {author} {\bibinfo {author} {\bibfnamefont {Dylan~J}\ \bibnamefont
  {Saunders}}, \bibinfo {author} {\bibfnamefont {Adam~J}\ \bibnamefont
  {Bennet}}, \bibinfo {author} {\bibfnamefont {Cyril}\ \bibnamefont
  {Branciard}}, \ and\ \bibinfo {author} {\bibfnamefont {Geoff~J}\ \bibnamefont
  {Pryde}},\ }\bibfield  {title} {\enquote {\bibinfo {title} {Experimental
  demonstration of non-bilocal quantum correlations},}\ }\href@noop {}
  {\bibfield  {journal} {\bibinfo  {journal} {arXiv preprint arXiv:1610.08514}\
  } (\bibinfo {year} {2016})}\BibitemShut {NoStop}%
\bibitem [{\citenamefont {Carvacho}\ \emph {et~al.}(2016)\citenamefont
  {Carvacho}, \citenamefont {Andreoli}, \citenamefont {Santodonato},
  \citenamefont {Bentivegna}, \citenamefont {Chaves},\ and\ \citenamefont
  {Sciarrino}}]{carvacho2016experimental}%
  \BibitemOpen
  \bibfield  {author} {\bibinfo {author} {\bibfnamefont {Gonzalo}\ \bibnamefont
  {Carvacho}}, \bibinfo {author} {\bibfnamefont {Francesco}\ \bibnamefont
  {Andreoli}}, \bibinfo {author} {\bibfnamefont {Luca}\ \bibnamefont
  {Santodonato}}, \bibinfo {author} {\bibfnamefont {Marco}\ \bibnamefont
  {Bentivegna}}, \bibinfo {author} {\bibfnamefont {Rafael}\ \bibnamefont
  {Chaves}}, \ and\ \bibinfo {author} {\bibfnamefont {Fabio}\ \bibnamefont
  {Sciarrino}},\ }\bibfield  {title} {\enquote {\bibinfo {title} {Experimental
  non-locality in a quantum network},}\ }\href@noop {} {\bibfield  {journal}
  {\bibinfo  {journal} {arXiv preprint arXiv:1610.03327}\ } (\bibinfo {year}
  {2016})}\BibitemShut {NoStop}%
\bibitem [{\citenamefont {Vandenberghe}\ and\ \citenamefont
  {Boyd}(1996)}]{vandenberghe1996semidefinite}%
  \BibitemOpen
  \bibfield  {author} {\bibinfo {author} {\bibfnamefont {Lieven}\ \bibnamefont
  {Vandenberghe}}\ and\ \bibinfo {author} {\bibfnamefont {Stephen}\
  \bibnamefont {Boyd}},\ }\bibfield  {title} {\enquote {\bibinfo {title}
  {Semidefinite programming},}\ }\href@noop {} {\bibfield  {journal} {\bibinfo
  {journal} {SIAM review}\ }\textbf {\bibinfo {volume} {38}},\ \bibinfo {pages}
  {49--95} (\bibinfo {year} {1996})}\BibitemShut {NoStop}%
\bibitem [{\citenamefont {von Prillwitz}(2015)}]{VonPrillwitz15MasterThesis}%
  \BibitemOpen
  \bibfield  {author} {\bibinfo {author} {\bibfnamefont {Kai}\ \bibnamefont
  {von Prillwitz}},\ }\emph {\bibinfo {title} {Statistical aspects of inferring
  Bayesian networks from marginal observations}},\ \href@noop {} {Master's
  thesis},\ \bibinfo  {school} {Fakult\"at f\"ur Mathematik und Physik der
  Albert-Ludwigs-Universit\"at Freiburg} (\bibinfo {year} {2015})\BibitemShut
  {NoStop}%
\bibitem [{\citenamefont {Cover}\ and\ \citenamefont
  {Thomas}(2012)}]{cover2012elements}%
  \BibitemOpen
  \bibfield  {author} {\bibinfo {author} {\bibfnamefont {Thomas~M}\
  \bibnamefont {Cover}}\ and\ \bibinfo {author} {\bibfnamefont {Joy~A}\
  \bibnamefont {Thomas}},\ }\href@noop {} {\emph {\bibinfo {title} {Elements of
  information theory}}}\ (\bibinfo  {publisher} {John Wiley \& Sons},\ \bibinfo
  {year} {2012})\BibitemShut {NoStop}%
\bibitem [{\citenamefont {Sch\"olkopf}\ and\ \citenamefont
  {Smola}(2002)}]{ScholkopfLearningWithKernels}%
  \BibitemOpen
  \bibfield  {author} {\bibinfo {author} {\bibfnamefont {B.}~\bibnamefont
  {Sch\"olkopf}}\ and\ \bibinfo {author} {\bibfnamefont {A.~J.}\ \bibnamefont
  {Smola}},\ }\href@noop {} {\emph {\bibinfo {title} {Learning with Kernels}}}\
  (\bibinfo  {publisher} {MIT Press},\ \bibinfo {year} {2002})\BibitemShut
  {NoStop}%
\bibitem [{\citenamefont {Gallavotti}(1999)}]{GallavottiStatisticalMechanics}%
  \BibitemOpen
  \bibfield  {author} {\bibinfo {author} {\bibfnamefont {G.}~\bibnamefont
  {Gallavotti}},\ }\href@noop {} {\emph {\bibinfo {title} {Statistical
  mechanics}}}\ (\bibinfo  {publisher} {Springer},\ \bibinfo {year}
  {1999})\BibitemShut {NoStop}%
\bibitem [{\citenamefont {Koller}\ and\ \citenamefont
  {Friedman}(2009)}]{KollerProbabilisticGraphicalModels}%
  \BibitemOpen
  \bibfield  {author} {\bibinfo {author} {\bibfnamefont {D.}~\bibnamefont
  {Koller}}\ and\ \bibinfo {author} {\bibfnamefont {N.}~\bibnamefont
  {Friedman}},\ }\href@noop {} {\emph {\bibinfo {title} {Probabilistic
  Graphical Models}}}\ (\bibinfo  {publisher} {MIT Press},\ \bibinfo {year}
  {2009})\BibitemShut {NoStop}%
\end{thebibliography}%

\end{document}